\documentclass{article}

\PassOptionsToPackage{numbers, compress}{natbib}

\usepackage[preprint]{neurips_2026}

\usepackage[utf8]{inputenc} % allow utf-8 input
\usepackage[T1]{fontenc}    % use 8-bit T1 fonts
\usepackage{hyperref}       % hyperlinks
\usepackage{url}            % simple URL typesetting
\usepackage{booktabs}       % professional-quality tables
\usepackage{amsfonts}       % blackboard math symbols
\usepackage{nicefrac}       % compact symbols for 1/2, etc.
\usepackage{microtype}      % microtypography        
\usepackage{amsmath}
\usepackage{amssymb}
\usepackage{amsthm}
\usepackage{graphicx}
\usepackage{subcaption}
\usepackage{multirow} 
\usepackage{makecell}
\usepackage{array}
\usepackage[table,dvipsnames]{xcolor}
\usepackage{capt-of}
\usepackage{enumitem}
\usepackage{xspace}
\usepackage{algorithm}
\usepackage{algpseudocode}
\usepackage{etoc}
\usepackage{longtable}
\usepackage{pifont}
\usepackage{ulem}
\usepackage{colortbl}
\usepackage{wrapfig}   
\newtheorem{theorem}{Theorem}[section]

\newtheorem{remark}{Remark}[section]

\newcommand{\pmstd}[2]{$#1\pm#2$}

\definecolor{OrangeVar}{RGB}{230,120,20}
\definecolor{BlueVar}{RGB}{30,100,200}

\newcommand{\ovar}[1]{{\color{OrangeVar}\boldsymbol{#1}}}
\newcommand{\bvar}[1]{{\color{BlueVar}\boldsymbol{#1}}}
\newcommand{\UniSTOK}{\textsc{UniStok}\xspace}

\definecolor{impcolor}{RGB}{180, 0, 0}
\definecolor{basegray}{RGB}{130, 130, 130}
\newcommand{\imp}[1]{{\color{impcolor}\textsuperscript{\scriptsize$\downarrow$#1}}}
\newcommand{\basenum}[1]{{\textcolor{basegray}{#1}}}

\newcommand{\pmstdbest}[2]{{\color{red}\textbf{#1}{\scriptsize$\pm$\textbf{#2}}}}

\title{Uniform Inductive Spatio-Temporal Kriging}

\author{
Lewei Xie$^{*}$ \\
City University of Hong Kong (Dongguan) \\
\texttt{72404204@cityu-dg.edu.cn}
\And
Haoyu Zhang$^{*}$ \\
City University of Hong Kong \\
\texttt{hzhang2838-c@my.cityu.edu.hk}
\And
Yulong Chen$^{*}$ \\
City University of Hong Kong (Dongguan) \\
\texttt{72405526@cityu-dg.edu.cn}
\And
Liangjun You \\
City University of Hong Kong (Dongguan) \\
\texttt{72540091@cityu-dg.edu.cn}
\And
Zongxian Yang \\
City University of Hong Kong (Dongguan) \\
\texttt{zongxian.yang@cityu-dg.edu.cn}
\And
Yifan Zhang \\
City University of Hong Kong (Dongguan) \\
\texttt{yif.zhang@cityu-dg.edu.cn}
}

\begin{document}

\maketitle
\begingroup
\renewcommand\thefootnote{\fnsymbol{footnote}}
\footnotetext[1]{Equal contribution.}
\endgroup

\begin{abstract}
Inductive spatio-temporal kriging infers signals at unobserved locations from observed sensors, but real-world observations are often incomplete and exhibit block-wise missingness caused by failures, interruptions, or maintenance.
A common impute-then-krige pipeline suffers from objective mismatch: better reconstruction on observed sensors does not necessarily improve downstream kriging, and value-dependent imputation bias can be propagated to unobserved nodes.
We propose \UniSTOK, a plug-and-play framework for inductive spatio-temporal kriging under incomplete observations.
We first introduce Reliability-guided Signal Regulation (RSR), which estimates entry-wise reliability from temporal continuity and spatial support, and uses it to regulate the input signals so that reliable observations are emphasized while long-gap or weakly supported entries are suppressed before spatial propagation.
We further introduce Residual Bias Calibration (RBC), which estimates value-conditioned residual prototypes after the main predictor converges and learns context-correction amplitudes to adaptively calibrate systematic over- or under-estimation in final kriging predictions.
Extensive experiments on real-world datasets show that \UniSTOK consistently improves multiple kriging backbones. 
\end{abstract}

\etocdepthtag.toc{main}
\section{Introduction}

Spatio-temporal applications rely on sensor deployments to collect raw observations for downstream tasks such as intelligent transportation \cite{shao2022pre, shao2022decoupled}, environmental monitoring \cite{ferchichi2025deep}, and urban analytics \cite{jin2023spatio, jiang2022graph}.  
In practice, however, the capital and maintenance costs of large-scale deployments \cite{wang2022inferring} make it infeasible to instrument every location with sensors.
With the recent development of Graph Neural Networks (GNNs) \cite{kipf2016semi}, inductive spatio-temporal kriging (ISK) \cite{appleby2020kriging, wu2021inductive} has been proposed as a promising paradigm for inferring signals at unobserved locations. 
It models the sensing system as a graph, using randomly sampled subgraphs for scalable learning.  

However, observed sensors in real-world deployments often provide incomplete observations because of equipment failures, communication interruptions, or maintenance \cite{li2020spatiotemporal}.
Such missingness can be substantial; for example, some loop sensors in Beijing's intelligent transportation system exhibit missing rates of 20\%--25\% \cite{qu2009ppca}.

This issue becomes more pronounced under structured missingness. 
Unlike randomly scattered missing values, block-wise missing patterns remove temporally contiguous observations, 
thereby disrupting local trends and temporal continuity that are difficult to recover from neighboring observations alone~\cite{du2024tsi}.
In ISK, this distinction becomes even more critical. 
Our preliminary analysis confirms this effect: under the same missing rate, block-wise missingness consistently causes larger kriging performance degradation than random missingness, as detailed in Appendix~\ref{app:block_missing_analysis}. 
Motivated by this observation, we study a practical yet underexplored setting: \textbf{\textit{spatio-temporal kriging under block-wise missingness on observed sensors.}}

A straightforward remedy is to complete the observed-sensor sequences via imputation before feeding them into ISK models.
However, this two-stage impute-then-krige pipeline suffers from an objective mismatch.
Recent studies have shown that lower reconstruction error does not necessarily translate into better downstream task performance \cite{wang2024task}.
Our motivating experiments reveal the same issue in ISK: this intuitive two-stage solution is unreliable.
Specifically, we observe weak or even negative correlations between imputation accuracy and downstream kriging performance, as detailed in Appendix~\ref{app:imputation_kriging_mismatch}.
These observations indicate that optimizing block-wise imputation as an isolated preprocessing task is not a reliable surrogate for improving downstream kriging.

A more fundamental issue is that imputation exhibit value-dependent residual bias.
Due to the imbalanced distribution of real-world sensor measurements, frequent regimes dominate the optimization objective, whereas peaks, valleys, and abrupt transitions are underrepresented.
As a result, imputation models tend to produce conservative estimates, leading to overestimation in low-value ranges and underestimation in high-value ranges.
Our experiments confirm this pattern: although a representative imputation model has a small global mean bias ($+1.27$), it still overestimates low-value regions ($+6.32$) and underestimates high-value regions ($-1.53$).
After kriging, this bias can be further amplified.
For example, with an ISK model, the low-value conditional bias increases from $5.35$ in the no-missing reference to $[11.57,12.20]$ under block-wise incomplete observations, while the high-value bias changes from $-1.85$ to $[-3.72,-3.42]$.
This suggests that the residual bias introduced during imputation is not isolated at the reconstruction stage; instead, it can be inherited and further reflected in downstream kriging predictions. More related results are provided in Appendix~\ref{app:bias_diagnosis}.
\textbf{\textit{Therefore, the key challenge is to calibrate value-dependent residual bias in kriging under incomplete observations.}} Figure~\ref{fig:intro_overview} summarizes the problem setting, the core challenges, and the design principle of our framework.

\begin{figure*}[t]
    \centering
    \includegraphics[width=\textwidth]{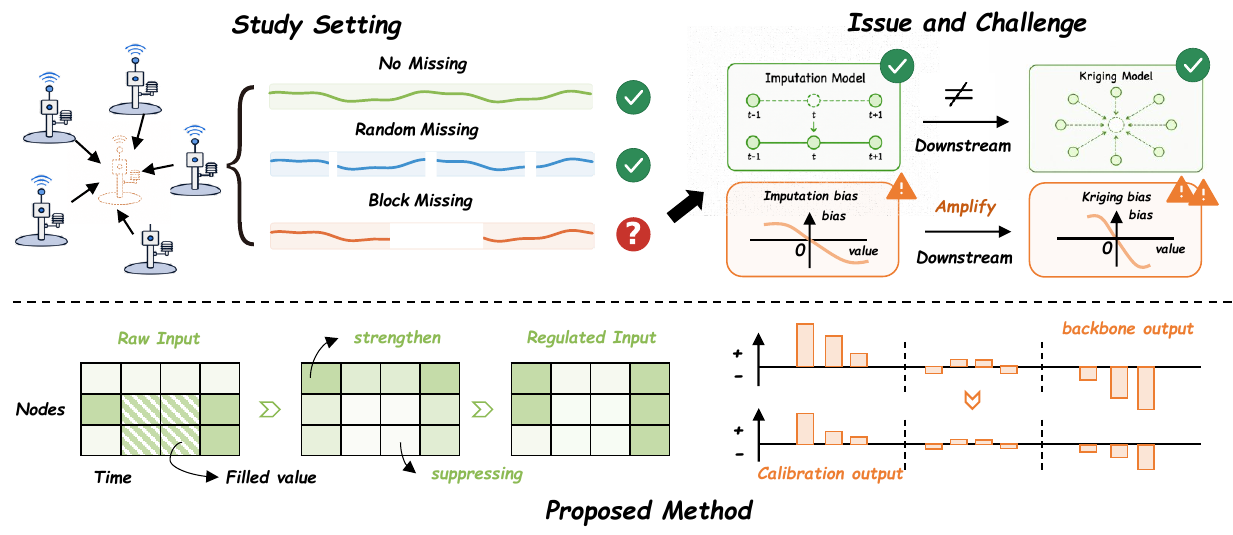}
    \caption{
    Overview of the studied setting, key challenge, and proposed solution.
    }
    \label{fig:intro_overview}
\end{figure*}

To address this challenge, we propose a framework with two complementary mechanisms.
\ding{182} \textbf{\textit{First}}, we suppress unreliable information propagation before kriging.
Directly correcting imputed entries from observed sensors is insufficient, because imputation-induced residual bias can be transformed and propagated by the kriging backbone through spatial aggregation and representation learning.
We therefore introduce \textit{Reliability-guided Signal Regulation} (RSR), which estimates entry-wise reliability from temporal continuity and spatial support and constructs a reliability-aware regulated input.
By enhancing reliable observations and suppressing weakly supported entries, RSR reduces the propagation of biased imputed values to unobserved nodes.
\ding{183} \textbf{\textit{Second}}, we calibrate the final kriging predictions with \textit{Residual Bias Calibration} (RBC).
Since value-dependent residual bias may still remain after RSR, RBC estimates value-conditioned residual prototypes from observed regions after the main predictor converges, where the main predictor consists of RSR and the shared kriging backbone.
RBC learns an adaptive correction amplitude for each prediction and selectively transfers residual prototypes to the infer of unobserved nodes.
This learnable calibration enables fine-grained residual adjustment and corrects overestimation or underestimation in the final outputs.

The main contributions of this paper are summarized as follows:

\begin{itemize} [leftmargin=*]
    \item We formulate spatio-temporal kriging under block-wise incomplete observations and propose \UniSTOK, a plug-and-play framework that enhances existing kriging backbones.

    \item We design \textit{Reliability-guided Signal Regulation} (RSR) to estimate temporal and spatial reliability and regulate incomplete observations, suppressing unreliable information propagation.

    \item We develop \textit{Residual Bias Calibration} (RBC), which learns value-conditioned residual prototypes and context-aware correction amplitudes to adaptively calibrate bias.

    \item We conduct extensive experiments on four real-world datasets, namely METR-LA, PEMS-BAY, NREL-AL, and NZ-Highway, showing that \UniSTOK consistently improves multiple state-of-the-art kriging backbones with lightweight plug-in overhead.
\end{itemize}

\section{Related Works}

Spatio-temporal data imputation aims to recover missing values in 
spatio-temporal sequences and has been widely used in traffic and 
environmental monitoring~\cite{yuan2022stgan}. Early statistical methods, 
including mean imputation, spline interpolation~\cite{mckinley1998cubic}, 
and regression~\cite{saad2020machine}, provide simple reconstruction 
strategies but are limited in capturing nonlinear temporal evolution and 
complex spatial dependencies. Recently, graph neural networks (GNNs) have 
become a major paradigm for modeling non-Euclidean spatial structures and 
temporal dynamics~\cite{wu2020comprehensive, zhang2025strapspatiotemporalpatternretrieval}. 
From the task perspective, GNN-based imputation can be broadly divided into 
in-sample and out-of-sample methods~\cite{wu2021inductive, jin2024survey}.

\textbf{\textit{In-Sample Imputation}}.
In-sample imputation focuses on recovering missing entries within the 
observed sensor set and the observed time range. Representative methods 
usually combine graph-based spatial modeling with temporal sequence 
modeling. For example, GACN~\cite{ye2021spatial} integrates Graph Attention 
Networks (GAT)~\cite{velivckovic2017graph} with temporal convolutions to 
capture spatial interactions and local temporal patterns. 
GRIN~\cite{cini2021filling} employs a spatio-temporal encoder based on 
Message Passing Neural Networks (MPNN)~\cite{gilmer2017neural} and Gated 
Recurrent Units (GRU)~\cite{chung2014empirical} through a two-stage 
imputation process. Bidirectional models such as 
DGCRIN~\cite{kong2023dynamic}, GARNN~\cite{shen2023bidirectional}, and 
MDGCN~\cite{liang2022memory} further exploit information from both temporal 
directions, while differing in their intermediate propagation and refinement 
workflows. More recently, STAMImputer~\cite{wang2025stamimputer} adopts a 
Mixture-of-Experts (MoE) framework with low-rank guided graph attention to 
dynamically generate spatial graphs. Despite their effectiveness, these 
methods mainly assume that the target nodes are already observed during 
training and inference, making them less suitable for imputing completely 
unobserved locations.

\textbf{\textit{Out-of-Sample Imputation}}.
Out-of-sample imputation, also known as spatio-temporal kriging, aims to 
infer values at unobserved locations from available observations. Unlike 
in-sample imputation, this setting requires the model to generalize beyond 
the observed sensor set, making spatial inductive capability essential. 
Early studies formulate the problem through graph-regularized tensor and 
matrix completion, which has been widely explored in spatio-temporal 
applications~\cite{bruna2013spectral, deng2016latent}. Recent work has 
shifted toward inductive spatio-temporal kriging (ISK) based on graph neural 
networks. KCN~\cite{appleby2020kriging} introduces graph convolution into 
kriging, while IGNNK~\cite{wu2021inductive} formulates kriging as graph 
signal reconstruction and enables inductive generalization across different 
observed and unobserved node subsets.
More recent ISK methods further improve generalization, virtual-node 
representation, and robustness to sparse graphs. KITS~\cite{xu2025kits} 
reduces the training--inference graph gap by incrementally inserting virtual 
nodes, pairing them with similar observed nodes, and constructing pseudo 
labels for virtual-node learning. STA-GANN~\cite{li2025sta} enhances 
generalizable spatio-temporal kriging with timestamp shifts, 
metadata-aware dynamic graph modeling, and adversarial learning for unseen 
sensors. DarkFarseer~\cite{liang2025darkfarseer} improves virtual-node 
representation under sparse and noisy graphs via hidden style enhancement, 
regional contrastive association, and graph denoising.

\section{Preliminaries}

\textbf{\textit{Spatio-temporal Kriging}}.
Let $\mathcal{G}=(\mathcal{V},\mathcal{E})$ denote a sensor graph with node set $\mathcal{V}$ and edge set $\mathcal{E}$.
We partition the nodes into observed and unobserved sets, $\mathcal{V}=\mathcal{V}_o\cup\mathcal{V}_u$ and $\mathcal{V}_o\cap\mathcal{V}_u=\emptyset$, where $|\mathcal{V}_o|=N$ and $|\mathcal{V}_u|=U$ are the numbers of observed and unobserved locations, respectively.
For a single feature channel, the measurements at observed locations over $T$ historical time steps are denoted by $\mathcal{X} \in \mathbb{R}^{N \times T}$, and the target signals at unobserved locations are denoted by $\mathcal{Y} \in \mathbb{R}^{U \times T}$.
The spatial structure of the entire graph is encoded by a weighted adjacency matrix $\mathcal{A} \in [0,1]^{(N+U)\times(N+U)}$, constructed using a thresholded Gaussian kernel over geographic distances~\cite{li2017diffusion}.
The goal of spatio-temporal kriging is to estimate $\mathcal{Y}$ from the observed measurements $\mathcal{X}$ and the graph structure $\mathcal{A}$.

\textbf{\textit{Inductive Spatio-temporal Kriging}}.
Each subgraph sample contains $n$ observed nodes selected from $\mathcal{V}_o$ and $u$ unobserved nodes selected from $\mathcal{V}_u$, where $n \leq N$ and $u \leq U$.
For a temporal window of length $t$, a subgraph sample is represented by $(X, A)$, where $X \in \mathbb{R}^{(n+u)\times t}$ denotes the windowed input matrix with unobserved-node entries masked, and $A \in [0,1]^{(n+u)\times(n+u)}$ is the corresponding subgraph adjacency matrix induced from $\mathcal{A}$.
Let $\mathcal{F}$ denote the kriging model, with output $\hat{Y}\in\mathbb{R}^{u\times t}$ on unobserved nodes; if all nodes are predicted, we select the unobserved rows.

\section{Methodology}

Figure~\ref{fig:framework} illustrates the overall framework of \UniSTOK.
Given incomplete observations, the observation mask, and the graph structure, \UniSTOK first applies Reliability-guided Signal Regulation (RSR) to construct a reliability-aware input view by estimating entry-wise reliability from temporal continuity and spatial support.
As a plug-and-play framework, \UniSTOK treats any kriging backbone as a shared predictor $F(\cdot)$ without modifying its internal architecture.
The original input and the RSR-regulated input are fed into the same backbone to produce a base prediction and a reliability-regulated prediction, which are adaptively fused by an entry-wise reliability-guided gate to obtain the main kriging output $\hat{Y}$.
Residual Bias Calibration (RBC) is then applied after the main predictor has converged.
Specifically, the main predictor, including RSR, the shared backbone, and the fusion, is fixed.
RBC then estimates value-dependent residual prototypes and learns context-aware correction amplitudes to refine $\hat{Y}$.
In this way, RBC calibrates value-dependent bias as a separate post-hoc module, while leaving the backbone unchanged and preserving the plug-and-play design.

\begin{figure*}[t]
    \centering
    \includegraphics[width=\textwidth]{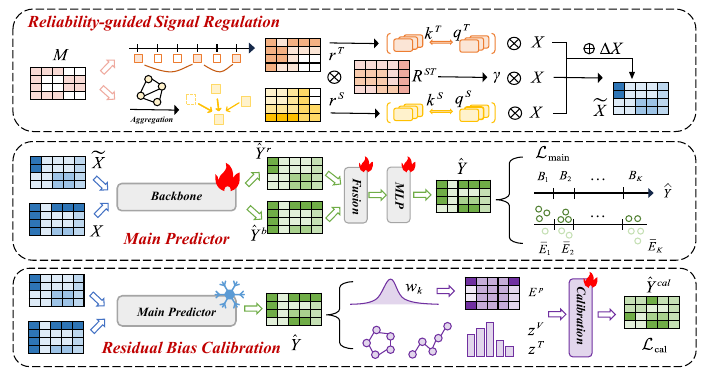}
    \caption{Overview of \UniSTOK, which consists of three components: Reliability-guided Signal Regulation (RSR), the shared-backbone main predictor, and Residual Bias Calibration (RBC).}
    \label{fig:framework}
\end{figure*}

\subsection{Reliability-guided Signal Regulation}

In spatio-temporal kriging, the input usually contains both real observations and imputed values for missing entries.
However, imputed values with long temporal gaps or weak spatial support may introduce unreliable patterns and bias subsequent spatio-temporal propagation.
To mitigate this issue, we propose \textit{Reliability-guided Signal Regulation} (RSR), which estimates entry-wise reliability over the input window and uses it to strengthen trustworthy signals while suppressing unreliable ones before they are propagated by the kriging backbone.

\textbf{\textit{Spatio-temporal reliability.}}
Let $M\in\{0,1\}^{(n+u)\times t}$ denote the observation mask, where $M_{i,\tau}=1$ if the value of node $i$ at time $\tau$ is available and $M_{i,\tau}=0$ otherwise.
RSR estimates entry-wise reliability by jointly considering temporal continuity and spatial support.
\ding{182} For temporal continuity, we measure the temporal distance from each entry to its nearest available observation on the same node.
Let $\delta_{i,\tau}$ denote the bidirectional nearest-observation distance of node $i$ at time $\tau$, computed as the minimum temporal distance to the nearest available entry before or after $\tau$.
The temporal reliability is defined as
$
r^{T}_{i,\tau}
=
1-
\frac{\log(1+\delta_{i,\tau})}{\log(1+t)} .
$
\ding{183} For spatial support, we measure whether an entry is supported by observed neighboring nodes at the same time step:
$
r^{S}_{i,\tau}
=
\frac{
\sum_{j=1}^{n+u} A_{ij}M_{j,\tau}
}{
\sum_{j=1}^{n+u} A_{ij}+\epsilon
}.
$
We combine the two sources into a spatio-temporal reliability:
$
R_{i,\tau}
=
\sqrt{r^{T}_{i,\tau}r^{S}_{i,\tau}+\epsilon}.
$
The following theorem establishes that the reliability $R$ is a well-defined, bounded, and monotone measure of observation quality, supporting its use as a signal regulation coefficient (see Appendix~\ref{app:proof_reliability} for the full proof).

\begin{theorem}[Boundedness and Monotonicity of the Reliability]
\label{thm:reliability_bound}
Let $R_{i,\tau}=\sqrt{r^{T}_{i,\tau}\cdot r^{S}_{i,\tau}+\epsilon}$ where $r^{T}_{i,\tau}\in[0,1]$ and $r^{S}_{i,\tau}\in[0,1]$. Then:
\begin{enumerate}
    \item (Boundedness) $\sqrt{\epsilon}\leq R_{i,\tau}\leq \sqrt{1+\epsilon}$ for all $i,\tau$.
    \item (Monotonicity) $R_{i,\tau}$ is strictly increasing in both $r^{T}_{i,\tau}$ and $r^{S}_{i,\tau}$.
    \item (Sensitivity decay) Under block missingness of length $L$ centered at $\tau$, the temporal reliability satisfies $r^{T}_{i,\tau}\leq 1-\frac{\log(1+\lfloor L/2\rfloor)}{\log(1+t)}$, which vanishes as $L\to t$.
\end{enumerate}
\end{theorem}

This theorem shows that entries deep inside long missing blocks receive low reliability, while entries near available observations retain high reliability.
The bounded range helps keep the downstream regulation coefficient $\gamma_{i,\tau}$ numerically stable.

\textbf{\textit{Attention residual.}}
The reliability $R$ explicitly describes observation quality, but it remains a compressed quality indicator rather than a complete representation. 
Therefore, RSR introduces a reliability-guided attention residual to extract implicit contextual correction signals. The residual extracts contextual correction signals from informative temporal and spatial sources.

Specifically, we derive temporal queries $q^{T}_{\tau}$ and spatial queries $q^{S}_{i}$ from aggregated temporal and spatial reliability summaries, respectively, and derive temporal keys $k^{T}_{\tau}$ and spatial keys $k^{S}_{i}$ from the availability mask $M$.
The queries indicate where reliability-guided correction is needed, while the keys indicate which temporal positions and nodes can provide reliable contextual information.
The temporal and spatial reliability-aware affinity scores are computed as
$
S^{T}_{\tau,\tau'}
=
q^{T}_{\tau}k^{T}_{\tau'},
S^{S}_{ij}
=
q^{S}_{i}k^{S}_{j}.
$
They quantify which observed contexts are more compatible with each target node-time position, and are used to extract residual evidence for signal regulation.
The temporal and spatial residual features are then aggregated as
$
X^{T}_{i,\tau}
=
\sum_{\tau'=1}^{t}
S^{T}_{\tau,\tau'}X_{i,\tau'},
X^{S}_{i,\tau}
=
\sum_{j=1}^{n+u}
S^{S}_{ij}X_{j,\tau},
$
Finally, the residual correction is obtained by.
\begin{equation}
\Delta X_{i,\tau}
=
\eta
\cdot
\tanh
\left(
\Phi_{\Delta}
\left(
[X^{T}_{i,\tau};X^{S}_{i,\tau}]
\right)
\right),
\end{equation}
where $\Phi_{\Delta}(\cdot)$ is a lightweight fusion network and $\eta$ controls the residual magnitude. 
The bounded residual prevents over-correction and makes $\Delta X$ act as a conservative refinement to the original signal rather than a dominant transformation.

\textbf{\textit{Signal regulation.}}
After obtaining the reliability and the attention residual, RSR constructs a regulated input by jointly applying explicit reliability and contextual residual refinement. 
The goal is to amplify entries with strong spatio-temporal evidence and attenuate entries that are likely dominated by missingness-induced artifacts.
We map the reliability to a scaling coefficient:
\begin{equation}
\gamma_{i,\tau}
=
1+
\alpha 
\cdot
\tanh
\left(
\frac{R_{i,\tau}-\mu_R}{s_R+\epsilon}
\right),
\end{equation}
where $\mu_R$ and $s_R$ are the mean and standard deviation of $R$, $\alpha$ controls the regulation range, and $\gamma_{i,\tau}\in[1-\alpha,1+\alpha]$. 
Thus, $\gamma_{i,\tau}$ amplifies entries whose reliability is above the average level and attenuates those below it.
The regulated input is then given by
\begin{equation}
\widetilde{X}_{i,\tau}
=
\gamma_{i,\tau}X_{i,\tau}
+
\Delta X_{i,\tau}.
\end{equation}
Here, $\gamma_{i,\tau}X_{i,\tau}$ performs explicit reliability-guided signal regulation, while $\Delta X_{i,\tau}$ injects bounded contextual residual correction. 
Together, these two terms keep $\widetilde{X}$ close to the original input scale while enhancing reliable observations and attenuating unreliable imputed values.

\textbf{\textit{Adaptive gate fusion}}.
The regulated input $\widetilde{X}$ provides a reliability-guided view, but directly replacing $X$ may make the model overly dependent on estimated reliability.
\UniSTOK therefore adopts a dual-view design, where both views are processed by the same kriging backbone $\mathcal{F}$ with shared parameters:
$
\hat{Y}^{b}=\mathcal{F}(X,A),
\hat{Y}^{r}=\mathcal{F}(\widetilde{X},A),
$
where $\hat{Y}^{b}$ and $\hat{Y}^{r}$ denote the base and reliability-regulated predictions, respectively.
This shared-backbone design incorporates reliability guidance while preserving compatibility with existing kriging backbones.

Although the two prediction views are complementary, a fixed fusion rule is insufficient under spatially and temporally varying missingness, where different node-time positions may rely on the two views to different extents.
Therefore, we introduce an entry-wise reliability-guided gate to adaptively combine the two predictions:
\begin{equation}
G
=
\sigma
\left(
\Phi_g
\left(
[\hat{Y}^{b};\hat{Y}^{r};R]
\right)
\right),
\quad
\hat{Y}^{g}
=
(1-G)\odot \hat{Y}^{b}
+
G\odot \hat{Y}^{r}.
\end{equation}
Here, $\Phi_g(\cdot)$ is a lightweight gating network and $G\in[0,1]^{u\times t}$ controls the entry-wise contribution of the regulated branch.
Conditioned on prediction disagreement and reliability, the gate learns when to trust the reliability-regulated view and when to fall back to the base view. 
Finally, the fused prediction is refined as
$
\hat{Y}
=
\Phi_o(\hat{Y}^{g}),
$
where $\Phi_o(\cdot)$ is a lightweight output MLP.

\subsection{Residual Bias Calibration}

Although reliability-guided fusion improves robustness under incomplete observations, the prediction may still exhibit stable, value-dependent residual bias.
We therefore introduce \textit{Residual Bias Calibration} (RBC), a post-hoc calibration module applied after the main predictor has converged.

\textbf{\textit{Residual prototype estimation.}}
Specifically, the main predictor, including reliability-guided signal regulation, the shared backbone, and gate-guided fusion, is optimized to produce the fused prediction $\hat{Y}$.
During training, we estimate value-conditioned residual prototypes from the same predictions used for the training loss.
For each valid position, the residual is defined as
\begin{equation}
E_{i,\tau}
=
\hat{Y}_{i,\tau}
-
Y_{i,\tau},
\quad (i,\tau)\in\Omega .
\end{equation}
where positive and negative values indicate overestimation and underestimation, respectively.
To capture value-dependent bias, we divide the range of predicted values into $K$ bins with centers $\{c_k\}_{k=1}^{K}$, and compute the mean residual in each bin:
\begin{equation}
\bar{E}_{k}
=
\frac{
\sum_{i,\tau} \mathbb{I}(\hat{Y}_{i,\tau}\in \mathcal{B}_k) E_{i,\tau}
}{
\sum_{i,\tau} \mathbb{I}(\hat{Y}_{i,\tau}\in \mathcal{B}_k)+\epsilon
},
\end{equation}
where $\mathcal{B}_k$ denotes the $k$-th bin.
The residual prototypes $\{\bar{E}_{k}\}_{k=1}^{K}$ summarize the empirical overestimation or underestimation tendency of the main predictor across different value ranges.
Details of residual prototype estimation are provided in Appendix~\ref{app:residual_prototype}.

\textbf{\textit{Soft residual retrieval.}}
After the main predictor has converged and the residual prototypes are estimated, we freeze the main predictor and train only the calibration module.
For each fixed prediction $\hat{Y}_{i,\tau}$, we first retrieve a prototype residual from the value-dependent residual prototypes.
Instead of using hard bin assignment, which may introduce discontinuities near bin boundaries, we compute soft weights according to the distance between $\hat{Y}_{i,\tau}$ and all bin centers:
\begin{equation}
w_{i,\tau,k}
=
\frac{
\exp\left(
-\frac{(\hat{Y}_{i,\tau}-c_k)^2}{2\sigma_c^2}
\right)
}{
\sum_{\ell=1}^{K}
\exp\left(
-\frac{(\hat{Y}_{i,\tau}-c_{\ell})^2}{2\sigma_c^2}
\right)
},
\end{equation}
where $\sigma_c$ is the kernel bandwidth controlling the smoothness of soft residual retrieval.
The retrieved prototype residual is then computed as
$
E^{p}_{i,\tau}
=
\sum_{k=1}^{K}
w_{i,\tau,k}\bar{E}_{k}.
$
This provides a smooth estimate of the value-dependent residual bias associated with $\hat{Y}_{i,\tau}$.

\textbf{\textit{Adaptive correction amplitude.}}
The retrieved prototype residual is an average correction within a prediction range, and thus remains a coarse estimate of value-dependent residual bias.
Since node-time positions with similar predicted values may correspond to different spatial states and temporal regimes, applying the same correction uniformly can be inaccurate.
We therefore learn an adaptive correction amplitude to adjust the retrieved prototype residual for each node-time position.

For each prediction $\hat{Y}_{i,\tau}$, we compute two standardized state features:
the node-wise state $z^{V}_{i,\tau}$ based on the historical statistics of node $i$, and the time-wise state $z^{T}_{i,\tau}$ based on the statistics of time position $\tau$.
Together with the magnitude of the retrieved prototype residual, they form the calibration context:
\begin{equation}
a_{i,\tau}
=
\Phi_a
\left(
z^{V}_{i,\tau},
z^{T}_{i,\tau},
|E^{p}_{i,\tau}|
\right),
\end{equation}
where
$
z^{V}_{i,\tau}=(\hat{Y}_{i,\tau}-\mu^{V}_{i})/(\sigma^{V}_{i}+\epsilon)
$
and
$
z^{T}_{i,\tau}=(\hat{Y}_{i,\tau}-\mu^{T}_{\tau})/(\sigma^{T}_{\tau}+\epsilon)
$.
Here, $\Phi_a(\cdot)$ is a lightweight MLP followed by a sigmoid function, producing $a_{i,\tau}\in[0,1]$.
The two standardized states indicate whether the current prediction deviates from node-specific or time-specific statistics, while $|E^{p}_{i,\tau}|$ reflects the strength of the retrieved value-conditioned bias.
Thus, $a_{i,\tau}$ adaptively preserves or dampens the prototype residual according to the local prediction state.
The calibrated output is obtained by applying the amplitude-controlled residual correction:
$
\hat{Y}^{cal}_{i,\tau}
=
\hat{Y}_{i,\tau}
-
a_{i,\tau}E^{p}_{i,\tau}.
$

\begin{theorem}[Consistency of Soft Residual Retrieval]
\label{thm:soft_retrieval}
Let $\bar{E}_k$ denote the empirical bin-wise residual prototype computed over $N_{\mathrm{tr}}$ training samples, and let $\beta(v) = \mathbb{E}[\hat{Y}_{i,\tau} - Y_{i,\tau} \mid \hat{Y}_{i,\tau} = v]$ denote the true conditional bias function. Define the soft-retrieved prototype residual
$$E^{p}_{i,\tau} = \sum_{k=1}^{K} w_{i,\tau,k}\,\bar{E}_k,$$
where the weights $w_{i,\tau,k}$ are the Gaussian kernel weights centered at bin centers $\{c_k\}$. Then, as $N_{\mathrm{tr}} \to \infty$ and $\sigma_c \to 0$ with $K \to \infty$ such that $\max_k |c_{k+1} - c_k| \to 0$, we have
$E^{p}_{i,\tau} \;\xrightarrow{p}\; \beta(\hat{Y}_{i,\tau}),$
i.e., the soft-retrieved residual converges in probability to the true value-dependent bias at the predicted value.
\end{theorem}

This theorem guarantees that the residual prototypes $\{\bar{E}_k\}$ are not arbitrary heuristics but consistent estimates of the true value-dependent bias $\beta(v)$. Consequently, the calibrated output $\hat{Y}^{cal}_{i,\tau} = \hat{Y}_{i,\tau} - a_{i,\tau}E^{p}_{i,\tau}$ with $a_{i,\tau} \to 1$ asymptotically removes the value-dependent bias, providing a principled foundation.

\textbf{\textit{Optimization objective.}}
The optimization proceeds in two stages with different objectives.
The main predictor, including RSR, the shared backbone, and reliability-guided fusion, is trained with the standard masked kriging loss:
\begin{equation}
\mathcal{L}_{\mathrm{main}}
=
\frac{
\sum_{i,\tau}
\Omega_{i,\tau}
\left|
\hat{Y}_{i,\tau}
-
Y_{i,\tau}
\right|
}{
\sum_{i,\tau}\Omega_{i,\tau}
+
\epsilon
},
\end{equation}
where $\Omega$ denotes the valid supervision mask.
After the main predictor converges and the residual prototypes are estimated, all parameters of the main predictor are fixed.
Only the calibration parameters are then optimized using a balanced bin-wise MAE objective~\cite{willmott2005advantages}:
\begin{equation}
\mathcal{L}_{\mathrm{cal}}
=
\frac{1}{K}
\sum_{k=1}^{K}
\frac{
\sum_{i,\tau}
\Omega_{i,\tau}
\mathbb{I}(\hat{Y}_{i,\tau}\in\mathcal{B}_k)
\left|
\hat{Y}^{cal}_{i,\tau}
-
Y_{i,\tau}
\right|
}{
\sum_{i,\tau}
\Omega_{i,\tau}
\mathbb{I}(\hat{Y}_{i,\tau}\in\mathcal{B}_k)
+
\epsilon
}.
\end{equation}
By averaging errors within and across prediction bins, this objective gives each prediction range equal weight and prevents calibration from being dominated by frequent value regions.

\section{Experiments}

This section evaluates \UniSTOK by answering five research questions:

\begin{itemize} [leftmargin=*, itemsep=4pt]
    \item \textbf{RQ1:}
    Can existing baselines be consistently improved when used as backbones in our framework?
    \item \textbf{RQ2:}
    Does our end-to-end framework outperform the two-stage imputation-then-kriging pipeline?
    \item \textbf{RQ3:}
    What is the contribution of each proposed module?
    \item \textbf{RQ4:}
    How sensitive is performance to key hyperparameters?
    \item \textbf{RQ5:}
    Does our framework achieve a favorable parameter--accuracy trade-off?
\end{itemize}

\subsection{Experimental Setup}

\textbf{\textit{Datasets}}.
We conduct experiments on four real-world datasets:
\textbf{METR-LA}~\cite{li2017diffusion}, \textbf{PEMS-BAY}~\cite{li2017diffusion}, \textbf{NREL-AL}~\cite{wu2021inductive}, and \textbf{NZ-Highway}~\cite{li2024high}.
NZ-Highway is a traffic volume dataset curated from New Zealand highway monitoring records, containing 61 stations with an original missing rate of 31.20\%.
More details of these datasets are provided in Appendix~\ref{app:datasets}.

\textbf{\textit{Baselines}.}
We compare \UniSTOK with nine representative baselines from three categories.
\ding{182} \textbf{\textit{Basic spatio-temporal models.}} DCRNN~\cite{li2017diffusion}, GMAN~\cite{zheng2020gman}, and STGCN~\cite{yu2017spatio}.
\ding{183} \textbf{\textit{Basic kriging models.}} IGNNK~\cite{wu2021inductive}, INCREASE~\cite{zheng2023increase}, and SATCN~\cite{wu2021spatial}.
\ding{184} \textbf{\textit{Advanced kriging models.}} DarkFarseer~\cite{liang2025darkfarseer}, KITS~\cite{xu2025kits}, and STA-GANN~\cite{li2025sta}.
Detailed descriptions of the baselines are provided in Appendix~\ref{app:baselines}.

\subsection{Performance Comparison (\textbf{RQ1})}

To answer \textbf{RQ1}, we instantiate \UniSTOK with each ISK method as the backbone.
Unless otherwise specified, we repeat each experiment with different random seeds and report the mean performance.
Table~\ref{tab:main_results} reports the performance of our framework across different ISK backbones.
\UniSTOK consistently improves almost all models across three datasets and metrics.
The gains are observed not only for classical forecasting backbones but also for recent kriging models. 
We further report results under random missingness, the no-missing setting, and NZ-Highway, a real-world block-missing dataset, in Appendix~\ref{app:additional_results}.

\begin{table*}[t]
\centering
\caption{
  Performance comparison of plug-in enhancement across backbones and datasets.
  Lower is better.
  \textcolor{basegray}{Gray}: base model results;
  \textbf{Black}: results with our plug-in (\textbf{+Ours}).
  {\color{red}$\downarrow$\!x}: relative improvement (\%) of \textbf{+Ours} over Base.
}
\label{tab:main_results}
\renewcommand{\tabcolsep}{3pt}
\renewcommand{\arraystretch}{1.05}
\resizebox{0.85\textwidth}{!}{%
\begin{tabular}{ll r | ccc | ccc | cc | c}
\toprule

%% ── Header ──────────────────────────────────────────────────────────────
\multicolumn{3}{c|}{\multirow{2}{*}{\scalebox{1.05}{\textbf{Models}}}}
  & \multicolumn{3}{c|}{\textbf{METR-LA}}
  & \multicolumn{3}{c|}{\textbf{PEMS-BAY}}
  & \multicolumn{2}{c|}{\textbf{NREL-AL}}
  & \multirow{2}{*}{\shortstack{\textbf{Avg.}\\\textbf{Improv.(\%)}}} \\
\cmidrule(lr){4-6}\cmidrule(lr){7-9}\cmidrule(lr){10-11}
\multicolumn{3}{c|}{}
  & \textbf{MAE} & \textbf{RMSE} & \textbf{MAPE}
  & \textbf{MAE} & \textbf{RMSE} & \textbf{MAPE}
  & \textbf{MAE} & \textbf{RMSE}
  & \\
\midrule

%% ── DCRNN ───────────────────────────────────────────────────────────────
\multirow{2}{*}{\shortstack[l]{\textbf{DCRNN}\\{\scriptsize ICLR'18}}}
  & \multirow{2}{*}{}
  & \textcolor{basegray}{Base}
  & \basenum{6.20}  & \basenum{9.89}   & \basenum{17.30}
  & \basenum{3.70}  & \basenum{6.76}   & \basenum{8.86}
  & \basenum{2.11}  & \basenum{4.23}
  & \textcolor{basegray}{--} \\
  & & \textbf{+Ours}
  & \textbf{5.93}\imp{4.4}   & \textbf{9.48}\imp{4.1}   & \textbf{16.34}\imp{5.5}
  & \textbf{3.50}\imp{5.4}   & \textbf{6.70}\imp{0.9}   & \textbf{8.71}\imp{1.7}
  & \textbf{2.08}\imp{1.4}   & \textbf{4.06}\imp{4.0}
  & \textbf{\color{red}3.4} \\
\midrule

%% ── STGCN ───────────────────────────────────────────────────────────────
\multirow{2}{*}{\shortstack[l]{\textbf{STGCN}\\{\scriptsize IJCAI'18}}}
  & \multirow{2}{*}{}
  & \textcolor{basegray}{Base}
  & \basenum{6.28}  & \basenum{10.36}  & \basenum{17.42}
  & \basenum{3.68}  & \basenum{6.62}   & \basenum{8.72}
  & \basenum{2.25}  & \basenum{4.34}
  & \textcolor{basegray}{--} \\
  & & \textbf{+Ours}
  & \textbf{5.76}\imp{8.3}   & \textbf{9.40}\imp{9.3}   & \textbf{16.30}\imp{6.4}
  & \textbf{3.40}\imp{7.6}   & \textbf{6.59}\imp{0.5}   & \textbf{8.43}\imp{3.3}
  & \textbf{1.97}\imp{12.4}  & \textbf{3.85}\imp{11.3}
  & \textbf{\color{red}7.4} \\
\midrule

%% ── GMAN ────────────────────────────────────────────────────────────────
\multirow{2}{*}{\shortstack[l]{\textbf{GMAN}\\{\scriptsize AAAI'20}}}
  & \multirow{2}{*}{}
  & \textcolor{basegray}{Base}
  & \basenum{6.44}  & \basenum{11.27}  & \basenum{17.98}
  & \basenum{4.01}  & \basenum{7.58}   & \basenum{9.12}
  & \basenum{1.71}  & \basenum{3.24}
  & \textcolor{basegray}{--} \\
  & & \textbf{+Ours}
  & \textbf{5.96}\imp{7.5}   & \textbf{9.50}\imp{15.7}  & \textbf{16.51}\imp{8.2}
  & \textbf{3.75}\imp{6.5}   & \textbf{6.99}\imp{7.8}   & \textbf{8.89}\imp{2.5}
  & \textbf{1.48}\imp{13.5}  & \textbf{3.04}\imp{6.2}
  & \textbf{\color{red}8.5} \\
\midrule

%% ── IGNNK ───────────────────────────────────────────────────────────────
\multirow{2}{*}{\shortstack[l]{\textbf{IGNNK}\\{\scriptsize AAAI'21}}}
  & \multirow{2}{*}{}
  & \textcolor{basegray}{Base}
  & \basenum{6.07}  & \basenum{9.56}   & \basenum{17.17}
  & \basenum{3.72}  & \basenum{6.73}   & \basenum{8.90}
  & \basenum{2.31}  & \basenum{4.49}
  & \textcolor{basegray}{--} \\
  & & \textbf{+Ours}
  & \textbf{5.54}\imp{8.7}   & \textbf{9.07}\imp{5.1}   & \textbf{15.58}\imp{9.3}
  & \textbf{3.45}\imp{7.3}   & \textbf{6.56}\imp{2.5}   & \textbf{8.25}\imp{7.3}
  & \textbf{1.92}\imp{16.9}  & \textbf{3.80}\imp{15.4}
  & \textbf{\color{red}9.1} \\
\midrule

%% ── SATCN ───────────────────────────────────────────────────────────────
\multirow{2}{*}{\shortstack[l]{\textbf{SATCN}\\{\scriptsize Arxiv'21}}}
  & \multirow{2}{*}{}
  & \textcolor{basegray}{Base}
  & \basenum{6.12}  & \basenum{9.84}   & \basenum{17.66}
  & \basenum{3.68}  & \basenum{6.82}   & \basenum{8.92}
  & \basenum{1.66}  & \basenum{3.17}
  & \textcolor{basegray}{--} \\
  & & \textbf{+Ours}
  & \textbf{5.93}\imp{3.1}   & \textbf{9.64}\imp{2.0}   & \textbf{16.97}\imp{3.9}
  & \textbf{3.65}\imp{0.8}   & \textbf{6.77}\imp{0.7}   & \textbf{8.86}\imp{0.7}
  & \textbf{1.39}\imp{16.3}  & \textbf{2.94}\imp{7.3}
  & \textbf{\color{red}4.4} \\
\midrule

%% ── INCREASE ────────────────────────────────────────────────────────────
\multirow{2}{*}{\shortstack[l]{\textbf{INCREASE}\\{\scriptsize WWW'23}}}
  & \multirow{2}{*}{}
  & \textcolor{basegray}{Base}
  & \basenum{5.93}  & \basenum{9.47}   & \basenum{17.09}
  & \basenum{3.77}  & \basenum{6.88}   & \basenum{8.89}
  & \basenum{2.11}  & \basenum{4.25}
  & \textcolor{basegray}{--} \\
  & & \textbf{+Ours}
  & \textbf{5.35}\imp{9.8}   & \textbf{9.00}\imp{5.0}   & \textbf{16.60}\imp{2.9}
  & \textbf{3.47}\imp{8.0}   & \textbf{6.74}\imp{2.0}   & \textbf{8.58}\imp{3.5}
  & \textbf{2.08}\imp{1.4}   & \textbf{3.96}\imp{6.8}
  & \textbf{\color{red}4.9} \\
\midrule

%% ── KITS ────────────────────────────────────────────────────────────────
\multirow{2}{*}{\shortstack[l]{\textbf{KITS}\\{\scriptsize AAAI'25}}}
  & \multirow{2}{*}{}
  & \textcolor{basegray}{Base}
  & \basenum{6.55}  & \basenum{11.38}  & \basenum{22.65}
  & \basenum{4.09}  & \basenum{7.67}   & \basenum{10.56}
  & \basenum{2.70}  & \basenum{5.07}
  & \textcolor{basegray}{--} \\
  & & \textbf{+Ours}
  & \textbf{5.90}\imp{9.9}   & \textbf{9.67}\imp{15.0}  & \textbf{17.62}\imp{22.2}
  & \textbf{4.05}\imp{0.9}   & \textbf{7.01}\imp{7.3}   & \textbf{9.53}\imp{9.8}
  & \textbf{2.62}\imp{3.0}   & \textbf{4.18}\imp{17.6}
  & \textbf{\color{red}10.7} \\
\midrule

%% ── STA-GANN ────────────────────────────────────────────────────────────
\multirow{2}{*}{\shortstack[l]{\textbf{STA-GANN}\\{\scriptsize CIKM'25}}}
  & \multirow{2}{*}{}
  & \textcolor{basegray}{Base}
  & \basenum{6.18}  & \basenum{9.81}   & \basenum{18.03}
  & \basenum{3.74}  & \basenum{6.72}   & \basenum{8.92}
  & \basenum{1.56}  & \basenum{3.05}
  & \textcolor{basegray}{--} \\
  & & \textbf{+Ours}
  & \textbf{5.70}\imp{7.8}   & \textbf{9.31}\imp{5.1}   & \textbf{16.48}\imp{8.6}
  & \textbf{3.51}\imp{6.1}   & \textbf{6.60}\imp{1.8}   & \textbf{8.51}\imp{4.6}
  & \textbf{1.53}\imp{1.9}   & \textbf{2.99}\imp{2.0}
  & \textbf{\color{red}4.7} \\
\midrule

%% ── DarkFarseer ─────────────────────────────────────────────────────────
\multirow{2}{*}{\shortstack[l]{\textbf{DarkFarseer}\\{\scriptsize AAAI'26}}}
  & \multirow{2}{*}{}
  & \textcolor{basegray}{Base}
  & \basenum{6.32}  & \basenum{10.70}  & \basenum{19.62}
  & \basenum{4.43}  & \basenum{8.62}   & \basenum{12.03}
  & \basenum{2.20}  & \basenum{4.09}
  & \textcolor{basegray}{--} \\
  & & \textbf{+Ours}
  & \textbf{5.53}\imp{12.5}  & \textbf{9.47}\imp{11.5}  & \textbf{16.46}\imp{16.1}
  & \textbf{4.15}\imp{6.3}   & \textbf{8.49}\imp{1.5}   & \textbf{11.63}\imp{3.3}
  & \textbf{1.96}\imp{10.9}  & \textbf{3.78}\imp{7.6}
  & \textbf{\color{red}8.7} \\

\midrule
%% ── Overall Avg. ────────────────────────────────────────────────────────
\rowcolor{blue!8}
\multicolumn{3}{c|}{\scalebox{1.02}{\textbf{Avg.\ Improv.\ (\%)}}}
  & \textbf{\color{red}8.0}  & \textbf{\color{red}8.1}  & \textbf{\color{red}9.2}
  & \textbf{\color{red}5.4}  & \textbf{\color{red}2.8}  & \textbf{\color{red}4.1}
  & \textbf{\color{red}8.6}  & \textbf{\color{red}8.7}
  & \textbf{\color{red}6.9} \\
\bottomrule
\end{tabular}%
}
\end{table*}

\subsection{Comparison with Two-stage Pipeline (\textbf{RQ2})}

\begin{wraptable}{r}{6cm} 
\vspace{-0.6cm}
\centering
\footnotesize
\caption{
  Comparison with a two-stage pre-imputation pipeline.
}
\label{table:comparison}
\renewcommand{\tabcolsep}{4pt}
\renewcommand{\arraystretch}{1.05}
\resizebox{\linewidth}{!}{%
\begin{tabular}{l r | ccc}
\toprule

\multicolumn{2}{c|}{\textbf{Model}}
  & \textbf{METR-LA} & \textbf{PEMS-BAY} & \textbf{NREL-AL} \\
\midrule

%% ── ImputeFormer + STA-GANN ─────────────────────────────────────────
\multirow{3}{*}{ImputeFormer+STA-GANN}
  & MAE   & \pmstd{6.04}{0.15} & \pmstd{3.84}{0.05} & \pmstd{1.91}{0.05} \\
  & RMSE  & \pmstd{9.34}{0.23} & \pmstd{6.76}{0.09} & \pmstd{3.95}{0.19} \\
  & MAPE  & \pmstd{18.11}{0.88} & \pmstd{9.14}{0.05} & \textcolor{basegray}{--} \\
\midrule

%% ── Ours ────────────────────────────────────────────────────────────
\multirow{3}{*}{\textbf{\UniSTOK}}
  & MAE   & \pmstdbest{5.84}{0.05} & \pmstdbest{3.37}{0.00} & \pmstdbest{1.76}{0.01} \\
  & RMSE  & \pmstdbest{9.22}{0.05} & \pmstdbest{6.09}{0.03} & \pmstdbest{3.41}{0.03} \\
  & MAPE  & \pmstdbest{16.86}{0.13} & \pmstdbest{8.32}{0.04} & -- \\
\bottomrule
\end{tabular}%
}
\end{wraptable}

To answer \textbf{RQ2}, we compare \UniSTOK with a two-stage pipeline that first imputes missing values on observed sensors and then performs kriging on unobserved sensors.
Specifically, we use ImputeFormer~\cite{nie2024imputeformer} for pre-imputation and feed the completed observations into STA-GANN as the downstream ISK model.
As shown in Table~\ref{table:comparison}, \UniSTOK consistently outperforms the two-stage pipeline.

\subsection{Ablation Study (\textbf{RQ3})}

To answer \textbf{RQ3}, we evaluate the contributions of the two core modules, RSR and RBC, by comparing the full model with three variants: w/o RSR, w/o RBC, and the original backbone.
As shown in Figure~\ref{fig:ablation_mae}, both modules contribute to \UniSTOK.
The original backbone exhibits the largest degradation, confirming the overall effectiveness of the framework.
Removing RBC leads to a larger performance drop than removing RSR, highlighting residual bias as a key bottleneck, while the degradation of w/o RSR confirms the need to suppress unreliable signals.
Detailed module-level ablations are provided in Appendix~\ref{app:additional_ablation}.

\begin{figure*}[t]
    \centering
    \includegraphics[width=\textwidth]{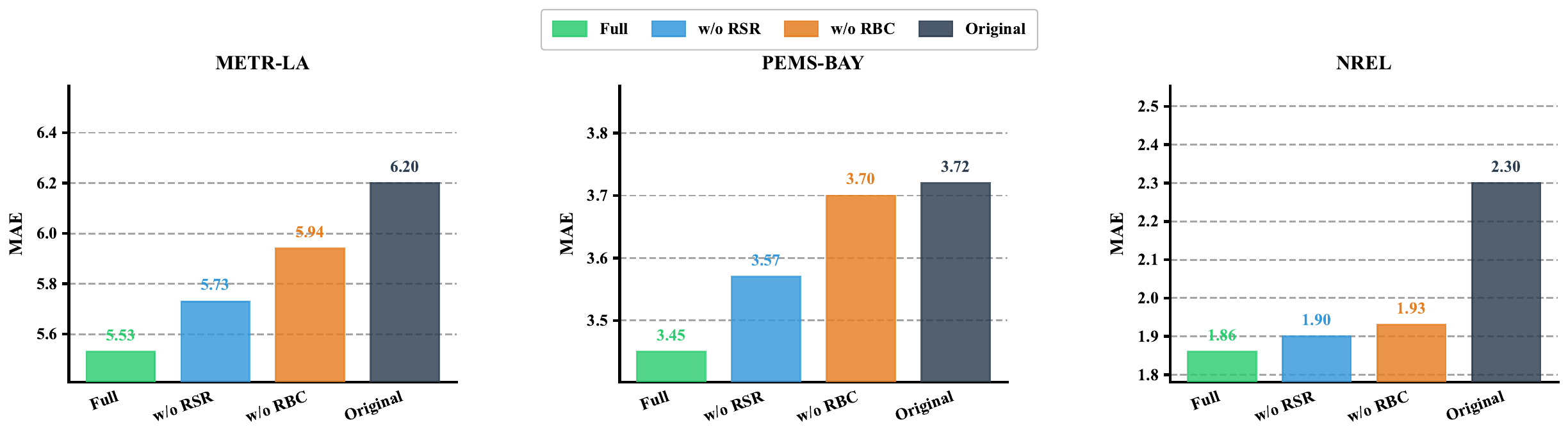}
    \caption{
    Ablation study of the core modules on three datasets.
    \textit{Full} denotes the complete \UniSTOK framework; \textit{w/o RSR} removes reliability-guide signal regulation; \textit{w/o RBC} removes residual bias calibration; \textit{Original} denotes the backbone without our framework.
    }
    \label{fig:ablation_mae}
\end{figure*}

\subsection{Hyperparameter Sensitivity (\textbf{RQ4})}

\begin{figure}[t]
    \centering
    \includegraphics[width=\textwidth]{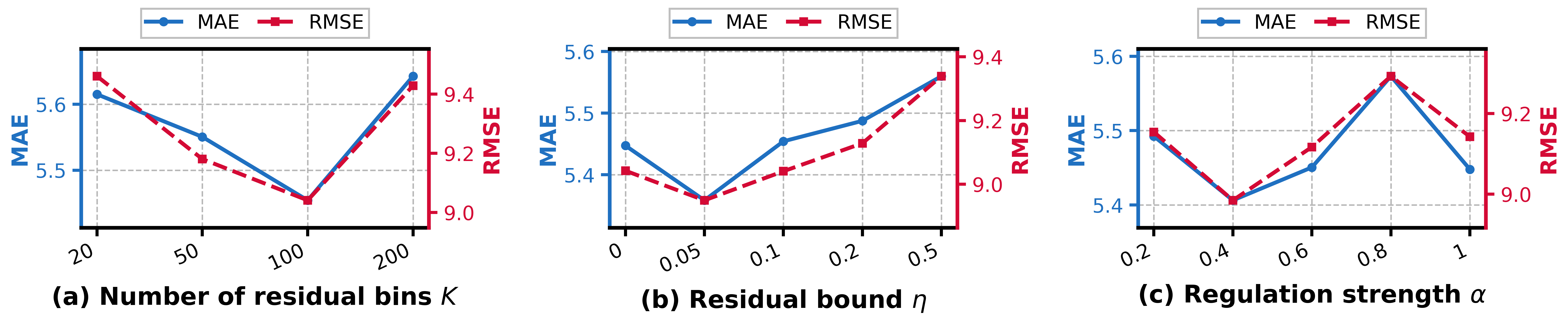}
    \caption{
    Hyperparameter sensitivity of \UniSTOK with IGNNK on METR-LA, varying the number of residual bins $K$, the residual correction bound $\eta$, and the reliability regulation strength $\alpha$.
    }
    \label{fig:hyperparameter_metr}
\end{figure}

To answer \textbf{RQ4}, we analyze the sensitivity of three key hyperparameters of \UniSTOK using IGNNK as the backbone on METR-LA.
As shown in Figure~\ref{fig:hyperparameter_metr}, \UniSTOK maintains stable performance across different hyperparameter settings.
A moderate number of residual bins enables more reliable residual prototype estimation, whereas an overly coarse or overly fine partition may weaken calibration.
For $\eta$ and $\alpha$, conservative settings are generally preferable on METR-LA, since overly strong residual correction or reliability regulation may introduce unnecessary perturbations to the kriging predictions.
These results suggest that, on METR-LA, \UniSTOK is not overly sensitive to the selected hyperparameters and benefits from moderate reliability regulation and bounded residual calibration.
Additional results on PEMS-BAY and NREL-AL are provided in Appendix~\ref{app:hyperparameter_sensitivity}.

\subsection{Efficiency \& Accuracy}

\begin{wrapfigure}{r}{6cm}
    \centering
    \includegraphics[width=0.88\linewidth]{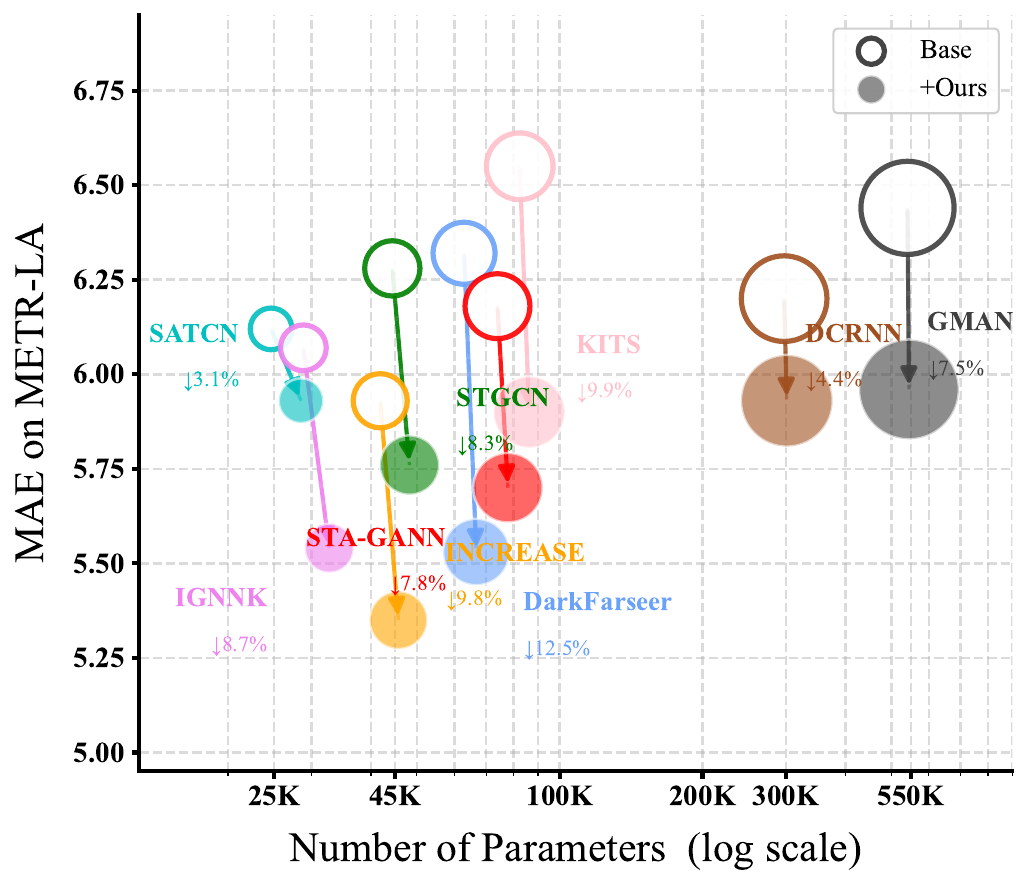}
    \caption{
    Parameter--accuracy trade-off on METR-LA.
    Empty and filled circles denote base backbones and their \UniSTOK-enhanced variants, respectively.
    }
    \label{fig:parameter_efficiency}
\end{wrapfigure}

To answer \textbf{RQ5}, Figure~\ref{fig:parameter_efficiency} compares the parameter--accuracy trade-off on METR-LA.
\UniSTOK introduces only lightweight plug-in components, including reliability regulation, adaptive fusion, and residual calibration, while keeping the backbone architecture unchanged.
As a result, the \textbf{+Ours} points remain nearly horizontally aligned with their corresponding base models, indicating negligible additional parameter cost relative to the original backbone size.
Meanwhile, all arrows consistently point downward, showing that \UniSTOK reduces MAE across backbones with substantially different parameter scales, from compact kriging models to larger attention-based architectures.
This suggests that the performance gain does not simply come from increased model capacity, but from the proposed reliability-aware regulation and residual bias calibration mechanisms.
Overall, \UniSTOK achieves a favorable parameter--accuracy trade-off.
Detailed parameter analysis is provided in Appendix~\ref{app:parameter_analysis}.

\section{Conclusion}
\label{sec:conclusion}

This paper studies spatio-temporal kriging under incomplete observations.
To address this problem, we propose \UniSTOK, a plug-and-play framework that combines Reliability-guided Signal Regulation for suppressing unreliable propagation with Residual Bias Calibration for correcting prediction bias.
Extensive experiments on real-world datasets show that \UniSTOK consistently improves multiple kriging backbones with negligible additional parameter cost.
Future work may extend \UniSTOK to multivariate signals, adaptive graph structures, and broader missingness scenarios.

\newpage

%=============================================================
%                       REFERENCES
%=============================================================
\small

\bibliographystyle{plainnat}
\bibliography{references}

\newpage

%%%%%%%%%%%%%%%%%%%%%%%%%%%%%%%%%%%%%%%%%%%%%%%%%%%%%%%%%%%%

\appendix
\etocdepthtag.toc{appendix}

\clearpage
\begingroup
\renewcommand{\contentsname}{Appendices}
\etocsettagdepth{main}{none}
\etocsettagdepth{appendix}{subsection}
\setcounter{tocdepth}{2}
\tableofcontents
\endgroup
\clearpage

\section{Notation}
\label{app:notation}

{\small
\begin{longtable}{@{\hspace{6pt}} p{3.4cm} p{9.8cm} @{\hspace{6pt}}}

\caption{Notations used in \UniSTOK.}
\label{tab:notation} \\

\toprule[1.2pt]
\textit{Notation} & \textit{Definition} \\
\midrule[0.6pt]
\endfirsthead

\multicolumn{2}{c}{\small\tablename~\thetable{} (continued)} \\[2pt]
\toprule[1.2pt]
\textit{Notation} & \textit{Definition} \\
\midrule[0.6pt]
\endhead

\bottomrule[1.2pt]
\endfoot

\bottomrule[1.2pt]
\endlastfoot

%% --- Data and Task ---
\rowcolor{gray!8}
\multicolumn{2}{@{\hspace{6pt}}l@{\hspace{6pt}}}%
  {\small\textsc{Data \& Task}} \\[1pt]

$\mathcal{G}=(\mathcal{V},\mathcal{E})$
  & Sensor graph with node set $\mathcal{V}$ and edge set $\mathcal{E}$ \\[2pt]

$\mathcal{V}_o$ / $\mathcal{V}_u$
  & Observed / unobserved node sets, with $|\mathcal{V}_o|=N$ and $|\mathcal{V}_u|=U$ \\[2pt]

$\mathcal{X}\in\mathbb{R}^{N\times T}$
  & Historical measurements at observed nodes over the full time horizon \\[2pt]

$\mathcal{Y}\in\mathbb{R}^{U\times T}$
  & Ground-truth signals at unobserved nodes \\[2pt]

$\mathcal{A}\in[0,1]^{(N+U)\times(N+U)}$
  & Weighted adjacency matrix of the whole sensor graph \\[2pt]

$n$ / $u$ / $t$
  & Number of sampled observed nodes / unobserved nodes / window length \\[2pt]

$X\in\mathbb{R}^{(n+u)\times t}$
  & Input signals of a sampled subgraph within one temporal window \\[2pt]

$A\in[0,1]^{(n+u)\times(n+u)}$
  & Adjacency matrix of the sampled induced subgraph \\[2pt]

$Y\in\mathbb{R}^{u\times t}$
  & Ground-truth target signals of sampled unobserved nodes \\[2pt]

$\hat{Y}\in\mathbb{R}^{u\times t}$
  & Predicted kriging output for unobserved nodes \\[2pt]

$M\in\{0,1\}^{(n+u)\times t}$
  & Observation mask, where $M_{i,\tau}=1$ indicates an observed entry \\[2pt]

$\mathcal{F}(\cdot)$
  & Backbone kriging model used in the shared predictor \\[3pt]

%% --- Missingness and Reliability ---
\rowcolor{gray!8}
\multicolumn{2}{@{\hspace{6pt}}l@{\hspace{6pt}}}%
  {\small\textsc{RELIABILITY-GUIDED SIGNAL REGULATION}} \\[1pt]

$\delta_{i,\tau}$
  & Bidirectional distance from entry $(i,\tau)$ to the nearest observed value of node $i$ \\[2pt]

$r^T_{i,\tau}$
  & Temporal reliability score derived from the missing span $\delta_{i,\tau}$ \\[2pt]

$r^S_{i,\tau}$
  & Spatial reliability score derived from observed neighboring support at time $\tau$ \\[2pt]

$R_{i,\tau}$
  & Spatio-temporal reliability score combining $r^T_{i,\tau}$ and $r^S_{i,\tau}$ \\[2pt]

$q^T_{\tau}$ / $q^S_i$
  & Temporal / spatial reliability queries derived from reliability scores \\[2pt]

$k^T_{\tau}$ / $k^S_i$
  & Temporal / spatial availability keys derived from the observation mask \\[2pt]

$S^T_{\tau,\tau'}$ / $S^S_{ij}$
  & Temporal / spatial reliability-aware affinity scores \\[2pt]

$X^T_{i,\tau}$ / $X^S_{i,\tau}$
  & Temporal / spatial residual evidence aggregated from reliable contexts \\[2pt]

$\Delta X_{i,\tau}$
  & Bounded attention residual used to refine the input signal \\[2pt]

$\eta$
  & Magnitude coefficient controlling the strength of the residual correction \\[2pt]

$\gamma_{i,\tau}$
  & Reliability-guided scale coefficient for signal regulation \\[2pt]

$\alpha$
  & Hyperparameter controlling the regulation range of $\gamma_{i,\tau}$ \\[2pt]

$\widetilde{X}_{i,\tau}$
  & Regulated input constructed by combining $\gamma_{i,\tau}X_{i,\tau}$ and $\Delta X_{i,\tau}$ \\[2pt]

$\hat{Y}^{b}$
  & Base prediction produced from the original input $X$ \\[2pt]

$\hat{Y}^{r}$
  & Reliability-regulated prediction produced from the regulated input $\widetilde{X}$ \\[2pt]

$G\in[0,1]^{u\times t}$
  & Entry-wise reliability-guided fusion gate \\[2pt]

$\hat{Y}^{g}$
  & Gate-fused prediction before the final output refinement \\[3pt]

%% --- Residual Bias Calibration ---
\rowcolor{gray!8}
\multicolumn{2}{@{\hspace{6pt}}l@{\hspace{6pt}}}%
  {\small\textsc{Residual Bias Calibration}} \\[1pt]

$E_{i,\tau}$
  & Prediction residual, defined as $E_{i,\tau}=\hat{Y}_{i,\tau}-Y_{i,\tau}$ \\[2pt]

$\mathcal{B}_k$ / $c_k$
  & The $k$-th prediction bin / its bin center \\[2pt]

$\bar{E}_k$
  & Mean residual prototype estimated within bin $\mathcal{B}_k$ \\[2pt]

$w_{i,\tau,k}$
  & Soft retrieval weight between prediction $\hat{Y}_{i,\tau}$ and bin center $c_k$ \\[2pt]

$E^p_{i,\tau}$
  & Soft-retrieved prototype residual for entry $(i,\tau)$ \\[2pt]

$z^V_{i,\tau}$ / $z^T_{i,\tau}$
  & Node-wise / time-wise standardized state feature for calibration \\[2pt]

$a_{i,\tau}$
  & Adaptive correction amplitude for the retrieved prototype residual \\[2pt]

$\hat{Y}^{cal}_{i,\tau}$
  & Calibrated prediction after amplitude-controlled residual correction \\[3pt]

%% --- Optimization ---
\rowcolor{gray!8}
\multicolumn{2}{@{\hspace{6pt}}l@{\hspace{6pt}}}%
  {\small\textsc{Optimization}} \\[1pt]

$\Omega\in\{0,1\}^{u\times t}$
  & Valid supervision mask for computing the kriging loss \\[1pt]

$\mathcal{L}_{\mathrm{main}}$
  & Masked MAE loss for training the main predictor \\[1pt]

$\mathcal{L}_{\mathrm{cal}}$
  & Balanced bin-wise MAE loss for training the calibration module \\[1pt]

$\mathbb{I}(\cdot)$
  & Indicator function \\

\end{longtable}
}

\section{Implementation Details}

\subsection{Datasets}
\label{app:datasets}

We provide a detailed description of the three real-world datasets used in our experiments.

\paragraph{METR-LA.}
METR-LA is a traffic speed dataset collected from loop detectors on the Los Angeles County highway network.
It contains traffic speed measurements from $207$ sensors over four months, from March 1, 2012 to June 30, 2012, with a sampling interval of $5$ minutes \cite{li2017diffusion}.

\paragraph{PEMS-BAY.}
PEMS-BAY is a traffic speed dataset collected by the California Performance Measurement System (PeMS) in the San Francisco Bay Area.
It contains traffic speed measurements from $325$ sensors, collected from January 1, 2017 to June 30, 2017, with a sampling interval of $5$ minutes~\cite{li2017diffusion}.

\paragraph{NREL-AL.}
NREL-AL is a solar power generation dataset derived from the National Renewable Energy Laboratory (NREL) solar power data for the Alabama region.
It records solar power outputs from $137$ photovoltaic power plants in Alabama in 2006, with a sampling interval of $5$ minutes~\cite{wu2021inductive}.

\paragraph{NZ-Highway.}
NZ-Highway is a traffic volume dataset curated by us from New Zealand highway monitoring records.
The raw data are collected from the New Zealand Transport Agency traffic monitoring system in the Wellington region.
After removing stations with extremely high missing rates, the processed dataset contains $61$ highway monitoring stations and $369{,}652$ time steps, spanning from January 1, 2013 to January 31, 2022, with a sampling interval of $15$ minutes.
The resulting dataset has an original missing rate of $31.20\%$, making it a suitable benchmark for evaluating spatio-temporal kriging under real-world incomplete observations. 

For NZ-Highway, we construct a distance-based adjacency matrix following~\cite{li2017diffusion}. 
Specifically, we first compute the geographical distance between each pair of sensors using the Haversine formula, and then convert the distance into an edge weight by a Gaussian kernel:
$
A_{ij} = \exp \left( - \left( \frac{\operatorname{dist}(v_i, v_j)}{\sigma} \right)^2 \right),
$
where $A_{ij}$ denotes the adjacency weight between sensors $v_i$ and $v_j$, $\operatorname{dist}(v_i, v_j)$ is the geographical distance between them, and $\sigma$ is the standard deviation of all non-zero pairwise distances.
To obtain a sparse and stable graph, we further apply a $k$-nearest-neighbor strategy with $k=10$.

\subsection{Imputation Methods}
\label{app:imputation_methods}

We provide a brief description of the seven imputation methods used in the imputation-first pipeline, covering simple filling strategies, traditional interpolation methods, and deep learning-based imputation models.

\textbf{Zero Filling}~\cite{du1994reduction}. 
Zero Filling replaces each missing value with zero in the normalized feature space, which corresponds to the mean-level value after standard normalization rather than the physical value zero.
It does not use temporal, spatial, or structural information and serves as a simple reference for evaluating.

\textbf{Mean Filling}~\cite{rothe1993mean}. 
Mean Filling replaces each missing value with the historical mean of the corresponding sensor.
This method preserves the node-level average scale but ignores short-term temporal variation and spatial dependency. It serves as a simple statistical baseline for completing block-missing observations.

\textbf{Linear Interpolation}~\cite{blu2004linear}. 
Linear Interpolation estimates missing values by linearly connecting the nearest available observations along the temporal dimension.
It is effective when missing values are randomly scattered and local temporal continuity is preserved.
However, under block missingness, the nearest valid observations may be far apart, making the interpolation less reliable for recovering local temporal trends.

\textbf{KNN Imputation}~\cite{zhang2012nearest}. 
KNN Imputation fills missing values using information from the nearest neighboring sensors.
It exploits spatial similarity among sensors and is therefore stronger than purely temporal filling methods when nearby sensors exhibit correlated dynamics.

\textbf{GRIN (ICLR 2022)}~\cite{cini2021multivariate}. 
GRIN is a graph recurrent imputation model designed for multivariate time series on graphs.
It combines graph neural networks with recurrent temporal modeling to propagate information across both spatial and temporal dimensions.
By explicitly exploiting graph structure, GRIN can recover missing sensor values using neighboring nodes and historical dynamics.

\textbf{SAITS (ESWA 2023)}~\cite{du2023saits}. 
SAITS is a self-attention-based imputation model for multivariate time series.
It uses Transformer-style attention to capture long-range temporal dependencies and feature interactions without recurrent computation.

\textbf{ImputeFormer (KDD 2024)}~\cite{nie2024imputeformer}. 
ImputeFormer is a Transformer-based spatio-temporal imputation model for incomplete traffic data.
It models temporal dynamics and spatial correlations jointly through attention-based representations, making it suitable for recovering missing values in sensor networks.

\subsection{Baselines}
\label{app:baselines}

We provide a brief description of the nine baseline models used in our experiments, organized by model category.

\textbf{DCRNN (ICLR 2018)}~\cite{li2017diffusion}. 
DCRNN models traffic dynamics by combining diffusion graph convolution with recurrent sequence modeling. 
It represents spatial dependency as bidirectional random walks on a directed graph and integrates the diffusion convolution into a gated recurrent unit, enabling temporal prediction over graph-structured sensor data. 
Although originally designed for forecasting, DCRNN can be adapted to kriging by using observed sensors as input nodes and evaluating predictions on held-out unobserved sensors.

\textbf{GMAN (AAAI 2020)}~\cite{zheng2020gman}. 
GMAN is an attention-based spatio-temporal forecasting model that jointly captures spatial correlations among sensors and temporal dependencies across time steps. 
It introduces spatial-temporal embedding and gated attention blocks to model dynamic dependencies without recurrent computation. 
In our setting, GMAN serves as a strong attention-based spatio-temporal backbone for assessing whether generic forecasting architectures can generalize to unobserved-node inference.

\textbf{STGCN (IJCAI 2018)}~\cite{yu2017spatio}. 
STGCN adopts a fully convolutional architecture for spatio-temporal graph modeling. 
It combines graph convolution for spatial dependency modeling with gated temporal convolution for sequence modeling, avoiding recurrent operations and enabling efficient training. 
We include STGCN as a classical spatio-temporal graph neural network baseline to evaluate the transferability of standard forecasting models to the kriging task.

\textbf{IGNNK (AAAI 2021)}~\cite{wu2021inductive}. 
IGNNK is one of the earliest graph neural network models for inductive spatio-temporal kriging. 
It formulates kriging as graph signal reconstruction and trains on sampled subgraphs so that the learned model can generalize to unseen nodes during inference. 
IGNNK uses graph convolutional operations to propagate information from observed nodes to unobserved nodes and reconstructs their temporal signals.

\textbf{INCREASE (WWW 2023)}~\cite{zheng2023increase}. 
INCREASE improves inductive kriging by introducing explicit relational candidates for unobserved nodes. 
Instead of relying only on direct geographic neighbors, it constructs multiple relation-aware neighbor sets and aggregates information from different spatial views. 
This design helps the model exploit richer contextual information when inferring signals at unseen locations.

\textbf{SATCN (2021)}~\cite{wu2021spatial}. 
SATCN, namely Spatial Aggregation and Temporal Convolution Networks, is designed for real-time spatio-temporal kriging. 
It combines spatial aggregation networks with temporal convolution networks, where the spatial module uses multiple aggregation operators to collect heterogeneous neighborhood information and the temporal module captures sequential dependencies. 
SATCN is included as a representative kriging-specific backbone with explicit spatial aggregation and temporal convolution.

\textbf{DarkFarseer (AAAI 2026)}~\cite{liang2025darkfarseer}. 
DarkFarseer is a recent inductive spatio-temporal kriging framework designed to improve virtual-node representation under sparse and noisy graph connections. 
It introduces hidden style enhancement to strengthen the representation of virtual nodes, contrastive learning to associate virtual-node patterns with regional components, and graph denoising to reduce the influence of unreliable edges. 
We include DarkFarseer as a recent strong baseline that explicitly targets representation quality and graph sparsity in ISK.

\textbf{KITS (AAAI 2025)}~\cite{xu2025kits}. 
KITS addresses the graph gap between training and inference in inductive kriging. 
Instead of only masking observed nodes during training, it incrementally inserts virtual nodes into the training graph so that the model can better match the inference scenario where unobserved nodes are present. 
It further pairs virtual nodes with similar observed nodes and constructs pseudo labels to provide more reliable supervision for virtual-node learning.

\textbf{STA-GANN (CIKM 2025)}~\cite{li2025sta}. 
STA-GANN is a recent graph adversarial framework for valid and generalizable spatio-temporal kriging. 
It models timestamp shifts with a decoupled phase module, updates spatial relations through dynamic metadata-aware graph modeling, and improves transferability to unknown sensors using adversarial learning. 
We include STA-GANN as an advanced baseline that focuses on dynamic spatial dependency modeling and generalization to unseen nodes.

\subsection{Experiment Details}
\label{app:experiment_details}

\paragraph{Evaluation protocol.}
We follow the kriging setting and split each dataset along the spatial dimension.
Nodes are divided into training, validation, and testing sets with a ratio of $70\%/10\%/20\%$.
Only training nodes are visible during optimization, while validation and testing nodes are held out as unseen locations.

\paragraph{Sampling protocol.}
For each sample, we randomly select a temporal window of length $h=24$ and construct a sampled subgraph.
The subgraph contains $N_{\mathrm{sub}}$ nodes, among which $u$ nodes are treated as unobserved targets.
During validation and testing, observed nodes are sampled from the training node set, while target nodes are sampled from the validation or testing node set.
We set $(N_{\mathrm{sub}},u)$ to $(110,10)$ for METR-LA, $(220,20)$ for PEMS-BAY, $(90,10)$ for NREL-AL, and $(30,5)$ for NZ-Highway.

\paragraph{Missingness.}
We simulate $20\%$ missingness on observed training nodes to evaluate robustness under incomplete observations.
Random missingness masks entries independently, block missingness masks connected node groups over continuous temporal intervals.
For NZ-Highway, we use its real missing observations and do not introduce additional simulated missingness.

\paragraph{Training hyperparameters.}
All models are trained with Adam for at most $200$ epochs, with batch size $64$ and initial learning rate $0.005$.
The learning rate decays by $0.5$ every $20$ epochs, and early stopping is applied based on validation MAE with patience $15$.
Each epoch contains $1000$ sampled training subgraphs, while validation and testing use $4000$ and $30000$ sampled subgraphs, respectively.

\paragraph{Backbone and plug-in settings.}
\UniSTOK is implemented as a non-intrusive plug-in framework.
It keeps the internal architecture of each backbone unchanged and only wraps the backbone with RSR, reliability-guided fusion, and RBC.
For most backbones, the hidden dimension is set to $64$.
For RSR, the reliability regulation strength is set to $\alpha=0.2$, and the attention residual is bounded with residual scale $0.05$.
For RBC, we use $K=100$ residual bins, initialize the soft-retrieval temperature as $\rho=2.0$, and optimize the calibration module with the balanced bin-wise MAE objective.

\paragraph{Evaluation metrics.}
We report MAE~\cite{willmott2005advantages}, RMSE~\cite{chai2014root}, and MAPE~\cite{hyndman2006another} after inverse normalization.

\paragraph{Computing resources.}
Experiments are conducted on one NVIDIA RTX 4090 GPU with $24$GB memory, $16$ vCPUs of Intel Xeon Gold 6430, and $120$GB system memory.
The software environment uses Python 3.12, PyTorch 2.5.1, CUDA 12.4, and Ubuntu 22.04.

\subsection{Parameter and Configuration Analysis}
\label{app:parameter_analysis}

This section reports the default configurations and trainable parameter counts of all evaluated backbones, followed by a parameter analysis of \UniSTOK.
All statistics are computed under the default setting with input and output window length $T=24$.
We report trainable parameters only.
For RBC, residual prototypes, bin statistics, and related residual tables are stored as non-trainable buffers and are therefore excluded from the trainable parameter count.

\paragraph{Backbone configurations.}

\begin{itemize} [leftmargin=0pt]
    \item \textbf{IGNNK.}
    We use hidden dimension $64$ and Chebyshev diffusion order $k=2$.
    The model contains three diffusion graph convolution layers with dimensions $24 \rightarrow 64 \rightarrow 64 \rightarrow 24$, resulting in $28{,}824$ trainable parameters.

    \item \textbf{STGCN.}
    We use hidden dimension $64$, bottleneck dimension $16$, Chebyshev order $k=3$, temporal kernel size $3$, and two spatial-temporal convolution blocks, resulting in $44{,}385$ trainable parameters.

    \item \textbf{DCRNN.}
    We use hidden dimension $64$, diffusion order $k=2$, two encoder layers, and two decoder layers, resulting in $297{,}281$ trainable parameters.

    \item \textbf{GMAN.}
    We use hidden dimension $64$, $8$ attention heads, three spatial-temporal attention blocks, and dropout rate $0.1$.
    Due to its stacked spatial-temporal attention and transform attention modules, GMAN has $540{,}993$ trainable parameters.

    \item \textbf{SATCN.}
    We use hidden dimension $32$, two SAN-TCN blocks, and temporal convolution kernel size $2$, resulting in $24{,}641$ trainable parameters.

    \item \textbf{INCREASE.}
    We use hidden dimension $64$ and top-$K=15$ relation-aware neighbor selection.
    It combines relation-aware aggregation with a context-aware GRU and contains $41{,}859$ trainable parameters.

    \item \textbf{KITS.}
    We use hidden dimension $64$, STGC order $1$, temporal stacking dimension $3$, and NCR weight $1.0$.
    Its input projection, three spatial-temporal graph convolution modules, hard-transfer module, and output projection contain $82{,}369$ trainable parameters.

    \item \textbf{STA-GANN.}
    We use hidden dimension $64$, two encoder-decoder layers, temporal kernel size $3$, dynamic adjacency top-$K=12$, and discriminator hidden dimension $64$.
    Its dynamic graph module, frequency-domain phase modeling, temporal encoder, and adversarial discriminator contain $73{,}948$ trainable parameters.

    \item \textbf{DarkFarseer.}
    We use hidden dimension $64$, two MPNN layers, BCC threshold $0.7$, contrastive weight $1.0$, and contrastive temperature $0.1$, resulting in $62{,}872$ trainable parameters.
\end{itemize}

\paragraph{Parameters of \UniSTOK.}
The additional trainable parameters of \UniSTOK come from three components: reliability-based signal regulation, adaptive pair fusion, and residual bias calibration.
Let $\Theta_{\mathrm{RSR}}$, $\Theta_{\mathrm{Fuse}}$, and $\Theta_{\mathrm{RBC}}$ denote their corresponding trainable parameters.
The total additional parameter count is
\begin{equation}
\Delta_{\UniSTOK}
=
|\Theta_{\mathrm{RSR}}|
+
|\Theta_{\mathrm{Fuse}}|
+
|\Theta_{\mathrm{RBC}}| .
\end{equation}
These parameters are shared across nodes and time steps, and therefore do not scale with the number of sensors $N$ or the number of training samples.
The trainable parts of \UniSTOK are implemented by lightweight affine projections and small MLPs.
For a two-layer MLP with input dimension $d_{\mathrm{in}}$, hidden dimension $d_h$, and output dimension $d_{\mathrm{out}}$, the parameter count is
$
P_{\mathrm{MLP}}(d_{\mathrm{in}}, d_h, d_{\mathrm{out}})
=
d_{\mathrm{in}}d_h + d_h
+
d_hd_{\mathrm{out}} + d_{\mathrm{out}} .
$
Thus, the plug-in overhead can be decomposed as
$
\Delta_{\UniSTOK}
=
P_{\mathrm{RSR}}
+
P_{\mathrm{gate}}
+
P_{\mathrm{fusion}}
+
P_{\mathrm{amp}},
$
where $P_{\mathrm{RSR}}$ is introduced by reliability-based signal regulation, $P_{\mathrm{gate}}$ and $P_{\mathrm{fusion}}$ are introduced by adaptive pair fusion, and $P_{\mathrm{amp}}$ is introduced by the correction-amplitude network in RBC.
The residual prototypes and bin-level statistics in RBC are non-trainable buffers and are not included in $\Delta_{\UniSTOK}$.
Under our default implementation, the total additional trainable parameter count is
$
\Delta_{\UniSTOK}
=
3{,}861 .
$

\subsection{Residual Prototype Estimation}
\label{app:residual_prototype}

This section details how RBC estimates residual prototypes from the training set.
The prototypes summarize value-dependent residual bias of the main predictor and are later used as fixed empirical statistics during residual bias calibration.
Validation data are used only for early stopping and model selection, not for residual estimation.

\paragraph{Training-time residual collection.}
Let $f_{\Theta}$ denote the main kriging predictor, including RSR, the shared backbone, and reliability-guided fusion.
During this stage, $f_{\Theta}$ is optimized by the main kriging objective.
After each epoch, the model is evaluated on the validation set to monitor convergence.
The best epoch is selected as
\begin{equation}
e^{\star}
=
\arg\min_{e}
\mathrm{MAE}^{(e)}_{\mathrm{val}},
\end{equation}
where $\mathrm{MAE}^{(e)}_{\mathrm{val}}$ is the validation MAE at epoch $e$.

Residual statistics are collected online using the predictions from the same forward pass used for the training loss.
For each valid training entry $(i,t)\in\Omega_{\mathrm{tr}}$, the residual at epoch $e$ is defined as
\begin{equation}
E^{(e)}_{i,t}
=
\hat{Y}^{(e)}_{i,t}
-
Y_{i,t},
\quad
(i,t)\in\Omega_{\mathrm{tr}} ,
\end{equation}
where $\Omega_{\mathrm{tr}}$ denotes the set of valid entries used for training supervision.
These residuals are used only to construct epoch-wise residual prototypes and do not introduce any additional optimization objective for the main predictor.

The prediction value range is divided into $K$ bins $\{\mathcal{B}_{k}\}_{k=1}^{K}$.
For each bin, we collect the training entries whose predictions fall into the corresponding value range:
\begin{equation}
\Omega^{(e)}_{k}
=
\left\{
(i,t)\in\Omega_{\mathrm{tr}}
\mid
\hat{Y}^{(e)}_{i,t}\in\mathcal{B}_{k}
\right\}.
\end{equation}
The epoch-wise residual prototype for the $k$-th bin is computed as
\begin{equation}
\bar{E}^{(e)}_{k}
=
\frac{
\sum_{(i,t)\in\Omega^{(e)}_{k}}
\left(
\hat{Y}^{(e)}_{i,t}-Y_{i,t}
\right)
}{
|\Omega^{(e)}_{k}|+\epsilon
},
\quad
k=1,\ldots,K .
\end{equation}
Thus, each epoch produces a residual table
\begin{equation}
\mathbf{R}^{(e)}
=
\left[
\bar{E}^{(e)}_{1},
\ldots,
\bar{E}^{(e)}_{K}
\right].
\end{equation}

\paragraph{Peak-weighted aggregation across epochs.}
Residual statistics from early epochs can be unstable because the predictor has not converged.
Conversely, using only the best epoch may make the prototype table sensitive to single-epoch noise.
Therefore, we aggregate the epoch-wise residual tables with a peak-weighted strategy centered at the validation-best epoch $e^{\star}$.
The weight of epoch $e$ is defined as
\begin{equation}
\omega_{e}
=
\beta^{|e-e^{\star}|},
\quad
0<\beta<1 .
\end{equation}
The final residual prototype of the $k$-th bin is obtained by
\begin{equation}
\bar{E}_{k}
=
\frac{
\sum_{e=1}^{E}
\omega_{e}
\bar{E}^{(e)}_{k}
}{
\sum_{e=1}^{E}
\omega_{e}
+\epsilon
},
\quad
k=1,\ldots,K .
\end{equation}
Equivalently, the final residual table is
\begin{equation}
\mathbf{R}
=
\frac{
\sum_{e=1}^{E}
\beta^{|e-e^{\star}|}
\mathbf{R}^{(e)}
}{
\sum_{e=1}^{E}
\beta^{|e-e^{\star}|}
+\epsilon
}.
\end{equation}
The resulting residual prototypes $\{\bar{E}_{k}\}_{k=1}^{K}$ summarize the empirical overestimation or underestimation tendency of the main predictor across different prediction ranges.
Importantly, $e^{\star}$ is determined by validation performance, whereas all residual values are estimated from the training set.
This separates model selection from residual estimation and prevents the calibrator from being driven by validation-set residuals.

\paragraph{Use in RBC.}
After residual prototypes are estimated, we reload the predictor parameters from the best epoch $e^{\star}$ and freeze the main predictor, including RSR, the shared backbone, and reliability-guided fusion.
The residual prototypes $\{\bar{E}_{k}\}_{k=1}^{K}$ are then treated as fixed empirical statistics.
RBC only learns how strongly these value-dependent residual corrections should be applied under different node-time contexts.
Given the frozen prediction $\hat{Y}_{i,t}$, RBC retrieves a differentiable prototype residual from $\{\bar{E}_{k}\}_{k=1}^{K}$ and applies a learnable correction to obtain $\hat{Y}^{\mathrm{cal}}_{i,t}$.

\begin{algorithm}[t]
\caption{Residual Prototype Estimation}
\label{alg:residual_prototype}
\small
\begin{algorithmic}[1]
\Require Training set $\mathcal{D}_{\mathrm{tr}}$, validation set $\mathcal{D}_{\mathrm{val}}$, main predictor $f_{\Theta}$, bins $\{\mathcal{B}_{k}\}_{k=1}^{K}$, peak factor $\beta$, patience $P$, maximum epochs $E$
\Ensure Best predictor parameters $\Theta^{\star}$, residual prototypes $\{\bar{E}_{k}\}_{k=1}^{K}$
\State Initialize $\Theta$; $m^{\star}\leftarrow\infty$; $e^{\star}\leftarrow 0$; $p\leftarrow 0$
\State Initialize epoch-wise residual table list $\mathcal{R}\leftarrow\emptyset$
\For{$e=1,\ldots,E$}
    \State Initialize residual sums $S_k\leftarrow 0$ and counts $C_k\leftarrow 0$ for $k=1,\ldots,K$
    \State Compute training predictions $\hat{\mathbf{Y}}^{(e)}=f_{\Theta}(\mathcal{D}_{\mathrm{tr}})$ and the main kriging loss $\mathcal{L}_{\mathrm{main}}$
    \For{each valid supervised entry $(i,t)\in\Omega_{\mathrm{tr}}$}
        \State Find bin index $k$ such that $\hat{Y}^{(e)}_{i,t}\in\mathcal{B}_{k}$
        \State $S_k\leftarrow S_k+\hat{Y}^{(e)}_{i,t}-Y_{i,t}$
        \State $C_k\leftarrow C_k+1$
    \EndFor
    \State Update $\Theta$ by minimizing $\mathcal{L}_{\mathrm{main}}$
    \For{$k=1,\ldots,K$}
        \State $\bar{E}^{(e)}_{k}\leftarrow S_k/(C_k+\epsilon)$
    \EndFor
    \State Append $\mathbf{R}^{(e)}=[\bar{E}^{(e)}_{1},\ldots,\bar{E}^{(e)}_{K}]$ to $\mathcal{R}$
    \State Evaluate $f_{\Theta}$ on $\mathcal{D}_{\mathrm{val}}$ and obtain validation score $m_e$
    \If{$m_e<m^{\star}$}
        \State $m^{\star}\leftarrow m_e$; $e^{\star}\leftarrow e$; save $\Theta^{\star}\leftarrow\Theta$; $p\leftarrow 0$
    \Else
        \State $p\leftarrow p+1$
    \EndIf
    \If{$p\geq P$}
        \State \textbf{break}
    \EndIf
\EndFor
\State \textbf{Peak-weighted aggregation}
\For{$k=1,\ldots,K$}
    \State $\bar{E}_{k}\leftarrow
    \frac{
    \sum_{e}\beta^{|e-e^{\star}|}\bar{E}^{(e)}_{k}
    }{
    \sum_{e}\beta^{|e-e^{\star}|}+\epsilon
    }$
\EndFor
\State \Return $\Theta^{\star}$, $\{\bar{E}_{k}\}_{k=1}^{K}$
\end{algorithmic}
\end{algorithm}

\newpage

\section{Theoretical Analysis}
\label{app:theoretical}

In this section, we provide rigorous mathematical proofs for the two core theoretical results stated in the main text. 
Theorem~\ref{thm:reliability_bound} establishes the foundational properties of the reliability field used in RSR, and Theorem~\ref{thm:soft_retrieval} formalizes the statistical consistency of the soft residual retrieval mechanism in RBC.

Throughout the proofs, we adopt a consistent color convention to facilitate tracking of key variables across derivation steps: {\color{OrangeVar}\textbf{orange}} highlights the \textit{temporal} reliability component and its derived quantities (e.g., $\ovar{r^T_{i,\tau}}$, $\ovar{\bar{E}_k}$, $\ovar{\gamma_{i,\tau}}$), which originate from the observation mask and missing span structure; {\color{BlueVar}\textbf{blue}} highlights the \textit{spatial or output} quantities (e.g., $\bvar{r^S_{i,\tau}}$, $\bvar{R_{i,\tau}}$, $\bvar{\beta(v)}$, $\bvar{E^p_{i,\tau}}$), which aggregate neighborhood or prediction-level information. This two-color scheme makes it straightforward to trace how temporal evidence and spatial evidence interact and propagate through each inequality chain.

\subsection{Proof of Theorem~\ref{thm:reliability_bound}: Boundedness and Monotonicity of the Reliability Field}
\label{app:proof_reliability}

\begin{proof}
We prove each claim in order.

\paragraph{\textit{Claim 1: Boundedness.}}
We begin by establishing the range of the two component reliability scores. By definition, the temporal reliability is
\begin{equation}
\ovar{r^{T}_{i,\tau}}
=
1-
\frac{\log(1+\delta_{i,\tau})}{\log(1+t)}.
\end{equation}
Since $\delta_{i,\tau}$ represents the minimum bidirectional distance to the nearest observed entry, we have $\delta_{i,\tau}\in\{0,1,\ldots,t-1\}$. When $\delta_{i,\tau}=0$ (the entry itself is observed), $\ovar{r^{T}_{i,\tau}}=1$. When $\delta_{i,\tau}$ achieves its maximum value $t-1$ (the entry is at the farthest point from any observation within the window), we have
\begin{equation}
\ovar{r^{T}_{i,\tau}}
=
1-\frac{\log(t)}{\log(1+t)}
=
\frac{\log(1+t)-\log(t)}{\log(1+t)}
=
\frac{\log(1+1/t)}{\log(1+t)}
>0.
\end{equation}
Therefore $\ovar{r^{T}_{i,\tau}}\in(0,1]$. More precisely, since $\log(1+\delta_{i,\tau})\leq \log(1+t-1)=\log(t)<\log(1+t)$, we always have $\ovar{r^{T}_{i,\tau}}>0$. Combined with the upper bound at $\delta_{i,\tau}=0$, we obtain $\ovar{r^{T}_{i,\tau}}\in(0,1]$.

For the spatial reliability, by definition:
\begin{equation}
\bvar{r^{S}_{i,\tau}}
=
\frac{
\sum_{j=1}^{n+u} A_{ij}M_{j,\tau}
}{
\sum_{j=1}^{n+u} A_{ij}+\epsilon
}.
\end{equation}
Since $A_{ij}\geq 0$, $M_{j,\tau}\in\{0,1\}$, and $\epsilon>0$, the numerator satisfies $0\leq \sum_j A_{ij}M_{j,\tau}\leq \sum_j A_{ij}$. Thus $\bvar{r^{S}_{i,\tau}}\in\bigl[0,\, \frac{\sum_j A_{ij}}{\sum_j A_{ij}+\epsilon}\bigr)\subset [0,1)$. In the limiting case where all neighbors are observed, $\bvar{r^{S}_{i,\tau}}$ approaches but does not reach 1 due to $\epsilon$. For practical purposes with small $\epsilon$, we treat $\bvar{r^{S}_{i,\tau}}\in[0,1]$.

Now consider the combined reliability field $R_{i,\tau}=\sqrt{\ovar{r^{T}_{i,\tau}}\cdot\bvar{r^{S}_{i,\tau}}+\epsilon}$. Since $\ovar{r^{T}_{i,\tau}}\in(0,1]$ and $\bvar{r^{S}_{i,\tau}}\in[0,1]$, the product $\ovar{r^{T}_{i,\tau}}\cdot\bvar{r^{S}_{i,\tau}}\in[0,1]$. Adding $\epsilon>0$ and taking the square root:
\begin{equation}
\sqrt{\epsilon}
\;\leq\;
R_{i,\tau}
=
\sqrt{\ovar{r^{T}_{i,\tau}}\cdot\bvar{r^{S}_{i,\tau}}+\epsilon}
\;\leq\;
\sqrt{1+\epsilon}.
\end{equation}
The lower bound $\sqrt{\epsilon}$ is achieved when $\bvar{r^{S}_{i,\tau}}=0$ (no observed neighbors), and the upper bound $\sqrt{1+\epsilon}$ is approached when both $\ovar{r^{T}_{i,\tau}}=1$ and $\bvar{r^{S}_{i,\tau}}=1$. This confirms that $R_{i,\tau}$ is strictly positive and bounded, ensuring numerical stability in all downstream computations.

\paragraph{\textit{Claim 2: Monotonicity.}}
We show that $R_{i,\tau}$ is strictly increasing in each component. Define $f(a,b)=\sqrt{ab+\epsilon}$ for $a\in(0,1]$ and $b\in[0,1]$. The partial derivative with respect to $a$ is:
\begin{equation}
\frac{\partial f}{\partial a}
=
\frac{b}{2\sqrt{ab+\epsilon}}.
\end{equation}
For any $b>0$, this derivative is strictly positive since $\sqrt{ab+\epsilon}>0$. Similarly:
\begin{equation}
\frac{\partial f}{\partial b}
=
\frac{a}{2\sqrt{ab+\epsilon}}
>0
\end{equation}
for all $a>0$. Since $\ovar{r^{T}_{i,\tau}}>0$ always holds (as shown in Claim 1), $R_{i,\tau}$ is strictly increasing in $\bvar{r^{S}_{i,\tau}}$ whenever $\bvar{r^{S}_{i,\tau}}>0$. Conversely, $R_{i,\tau}$ is strictly increasing in $\ovar{r^{T}_{i,\tau}}$ whenever $\bvar{r^{S}_{i,\tau}}>0$.

This monotonicity property is essential for the regulation mechanism: it guarantees that entries with stronger temporal continuity \emph{and} spatial support always receive higher reliability scores, ensuring a consistent ordering that the downstream regulation coefficient $\gamma_{i,\tau}$ can faithfully amplify.

\paragraph{\textit{Claim 3: Sensitivity decay under block missingness.}}
Consider a contiguous block of missing values of length $L$ centered at time $\tau$ for node $i$. The nearest observed entry is at distance at least $\lfloor L/2\rfloor$ from $\tau$, so $\delta_{i,\tau}\geq \lfloor L/2\rfloor$. Since $\ovar{r^{T}_{i,\tau}}$ is monotonically decreasing in $\delta_{i,\tau}$ (the logarithm is increasing), we obtain:
\begin{equation}
\ovar{r^{T}_{i,\tau}}
=
1-\frac{\log(1+\delta_{i,\tau})}{\log(1+t)}
\;\leq\;
1-\frac{\log(1+\lfloor L/2\rfloor)}{\log(1+t)}.
\end{equation}
As $L\to t$ (the block spans the entire window), $\lfloor L/2\rfloor \to \lfloor t/2\rfloor$, and:
\begin{equation}
\ovar{r^{T}_{i,\tau}}
\;\leq\;
1-\frac{\log(1+\lfloor t/2\rfloor)}{\log(1+t)}
\;\xrightarrow{t\to\infty}\;
1-\frac{\log(t/2)}{\log(t)}
=
\frac{\log 2}{\log t}
\;\to\; 0.
\end{equation}
For typical window sizes (e.g., $t=12$ or $t=24$), this yields $\ovar{r^{T}_{i,\tau}}\leq 0.32$ or $\ovar{r^{T}_{i,\tau}}\leq 0.29$, respectively. This demonstrates that the temporal reliability effectively penalizes entries embedded in long missing blocks, with the penalty growing logarithmically with block length.

Consequently, the combined reliability $R_{i,\tau}=\sqrt{\ovar{r^{T}_{i,\tau}}\cdot\bvar{r^{S}_{i,\tau}}+\epsilon}$ also decays, since $R_{i,\tau}\leq \sqrt{\ovar{r^{T}_{i,\tau}}+\epsilon}$ (using $\bvar{r^{S}_{i,\tau}}\leq 1$). Under severe block missingness where both temporal and spatial evidence are degraded simultaneously, the reliability field approaches its minimum $\sqrt{\epsilon}$, triggering maximal attenuation in the regulation coefficient $\gamma_{i,\tau}$.
\end{proof}

\begin{remark}
The logarithmic decay rate $\log(1+\delta)/\log(1+t)$ is a deliberate design choice. A linear decay $\delta/t$ would penalize entries too aggressively near block boundaries, while a sub-logarithmic decay (e.g., $\sqrt{\delta/t}$) would be insufficiently sensitive to long blocks. The logarithmic form provides a balanced sensitivity profile: rapid initial decay for short missing spans (where interpolation is still feasible) and gradual saturation for very long blocks (where the entry is essentially uninformative regardless of exact block length).
\end{remark}

\subsection{Proof of Theorem~\ref{thm:soft_retrieval}: Consistency of Soft Residual Retrieval}
\label{app:proof_src}

We now prove that the soft residual retrieval in RBC converges to the true conditional bias function as the number of training samples grows and the bin resolution increases. The proof proceeds in three stages: we first establish the almost sure convergence of each empirical bin prototype to its population counterpart, then analyze the approximation error introduced by the finite bin resolution, and finally combine both results to obtain convergence in probability of the soft-retrieved residual to the true conditional bias.

\begin{proof}

\noindent\textbf{\textit{Stage 1: Convergence of empirical bin prototypes.}}

Let $\mathcal{B}_k = [c_k - h, c_k + h)$ denote the $k$-th bin of width $2h$ centered at $c_k$, where $h = \max_k |c_{k+1} - c_k|/2$ controls the bin resolution. For a fixed bin $\mathcal{B}_k$, define the population-level conditional bias:
\begin{equation}
\beta_k
=
\mathbb{E}\!\left[\,\hat{Y}_{i,\tau} - Y_{i,\tau} \;\middle|\; \hat{Y}_{i,\tau} \in \mathcal{B}_k\right]
=
\frac{\mathbb{E}\!\left[\,(\hat{Y}_{i,\tau} - Y_{i,\tau})\,\mathbb{I}(\hat{Y}_{i,\tau}\in\mathcal{B}_k)\right]}{\mathbb{P}(\hat{Y}_{i,\tau}\in\mathcal{B}_k)}.
\end{equation}
The empirical prototype $\ovar{\bar{E}_k}$ computed over $N_{\mathrm{tr}}$ training samples is:
\begin{equation}
\ovar{\bar{E}_k}
=
\frac{
\sum_{i,\tau} \mathbb{I}(\hat{Y}_{i,\tau}\in \mathcal{B}_k)\,(\hat{Y}_{i,\tau} - Y_{i,\tau})
}{
\sum_{i,\tau} \mathbb{I}(\hat{Y}_{i,\tau}\in \mathcal{B}_k) + \epsilon
}.
\end{equation}
Let $N_k = \sum_{i,\tau}\mathbb{I}(\hat{Y}_{i,\tau}\in\mathcal{B}_k)$ denote the number of training samples falling into bin $\mathcal{B}_k$. Define the individual residual $e_{i,\tau} = \hat{Y}_{i,\tau} - Y_{i,\tau}$ and write $\ovar{\bar{E}_k} = \frac{1}{N_k+\epsilon}\sum_{i,\tau}\mathbb{I}(\hat{Y}_{i,\tau}\in\mathcal{B}_k)\,e_{i,\tau}$.

We apply the Strong Law of Large Numbers (SLLN) to the numerator and denominator separately. Provided that $\mathbb{E}[|e_{i,\tau}|] < \infty$ (which holds under any finite-variance prediction model) and that $\mathbb{P}(\hat{Y}_{i,\tau}\in\mathcal{B}_k) = p_k > 0$, as $N_{\mathrm{tr}}\to\infty$ we have:
\begin{equation}
\frac{1}{N_{\mathrm{tr}}}\sum_{i,\tau}\mathbb{I}(\hat{Y}_{i,\tau}\in\mathcal{B}_k)\,e_{i,\tau}
\;\xrightarrow{a.s.}\;
\mathbb{E}\!\left[e_{i,\tau}\,\mathbb{I}(\hat{Y}_{i,\tau}\in\mathcal{B}_k)\right],
\qquad
\frac{N_k}{N_{\mathrm{tr}}}
\;\xrightarrow{a.s.}\;
p_k.
\end{equation}
Since $p_k > 0$ by assumption, the continuous mapping theorem applied to the ratio gives:
\begin{equation}
\ovar{\bar{E}_k}
\;\xrightarrow{a.s.}\;
\frac{\mathbb{E}\!\left[e_{i,\tau}\,\mathbb{I}(\hat{Y}_{i,\tau}\in\mathcal{B}_k)\right]}{p_k}
=
\beta_k,
\end{equation}
where the last equality follows from the definition of conditional expectation. This establishes that each empirical prototype $\ovar{\bar{E}_k}$ is a strongly consistent estimator of the bin-level population bias $\beta_k$.

The almost sure convergence here is the key statistical guarantee: it means that for any fixed bin, the empirical prototype computed from training data will eventually be arbitrarily close to the true conditional bias within that bin, regardless of the specific realization of the training set. This justifies the use of $\ovar{\bar{E}_k}$ as a reliable proxy for the systematic prediction error in the $k$-th value range.

\noindent\textbf{\textit{Stage 2: Approximation error of the discretized bias.}}

The bin-level prototype $\beta_k$ approximates the true pointwise conditional bias function $\bvar{\beta(v)} = \mathbb{E}[e_{i,\tau} \mid \hat{Y}_{i,\tau} = v]$ only up to the resolution of the binning. We now quantify this approximation error.

Assume that $\bvar{\beta(v)}$ is Lipschitz continuous with constant $\Lambda > 0$, i.e., $|\bvar{\beta(v)} - \bvar{\beta(v')}| \leq \Lambda|v - v'|$ for all $v, v'$ in the prediction range. This is a mild regularity condition that holds whenever the conditional bias varies smoothly with the predicted value, which is expected in practice since the value-dependent bias of a trained neural network is a smooth function of its output.

For any $v \in \mathcal{B}_k$, the distance from $v$ to the bin center $c_k$ is at most $h$. Therefore:
\begin{equation}
|\beta_k - \bvar{\beta(v)}|
=
\left|\mathbb{E}[\bvar{\beta(\hat{Y}_{i,\tau})} \mid \hat{Y}_{i,\tau}\in\mathcal{B}_k] - \bvar{\beta(v)}\right|
\leq
\mathbb{E}\!\left[|\bvar{\beta(\hat{Y}_{i,\tau})} - \bvar{\beta(v)}| \;\middle|\; \hat{Y}_{i,\tau}\in\mathcal{B}_k\right]
\leq
\Lambda h,
\end{equation}
where the first inequality uses Jensen's inequality for the absolute value, and the second uses the Lipschitz condition together with $|\hat{Y}_{i,\tau} - v| \leq 2h$ for $\hat{Y}_{i,\tau}, v \in \mathcal{B}_k$.

This bound shows that the discretization error is $O(h)$: as the bin width $h \to 0$ (equivalently, $K \to \infty$ with $\max_k|c_{k+1}-c_k| \to 0$), the bin prototype $\beta_k$ converges uniformly to the true pointwise bias $\bvar{\beta(v)}$ for all $v$ in the prediction range. The Lipschitz assumption is the minimal regularity needed to guarantee this uniform convergence.

\noindent\textbf{\textit{Stage 3: Convergence of the soft-retrieved residual.}}

We now combine the two stages to analyze the soft-retrieved prototype residual:
\begin{equation}
\bvar{E^{p}_{i,\tau}}
=
\sum_{k=1}^{K}
w_{i,\tau,k}\,\ovar{\bar{E}_k},
\qquad
w_{i,\tau,k}
=
\frac{
\exp\!\left(-\frac{(\hat{Y}_{i,\tau}-c_k)^2}{2\sigma_c^2}\right)
}{
\sum_{\ell=1}^{K}\exp\!\left(-\frac{(\hat{Y}_{i,\tau}-c_\ell)^2}{2\sigma_c^2}\right)
}.
\end{equation}
The weights $\{w_{i,\tau,k}\}$ form a probability simplex (i.e., $w_{i,\tau,k}\geq 0$ and $\sum_k w_{i,\tau,k}=1$), so $\bvar{E^{p}_{i,\tau}}$ is a convex combination of the empirical prototypes. We decompose the total error as:
\begin{equation}
\left|\bvar{E^{p}_{i,\tau}} - \bvar{\beta(\hat{Y}_{i,\tau})}\right|
\leq
\underbrace{\left|\sum_{k=1}^{K} w_{i,\tau,k}\,\ovar{\bar{E}_k} - \sum_{k=1}^{K} w_{i,\tau,k}\,\beta_k\right|}_{\text{Term I: estimation error}}
+
\underbrace{\left|\sum_{k=1}^{K} w_{i,\tau,k}\,\beta_k - \bvar{\beta(\hat{Y}_{i,\tau})}\right|}_{\text{Term II: approximation error}}.
\end{equation}
This triangle inequality decomposition separates the statistical error (Term I, due to finite training data) from the structural error (Term II, due to finite bin resolution), allowing us to control each independently.

\textbf{Bounding Term I.}
Since $\{w_{i,\tau,k}\}$ is a probability simplex and $\sum_k w_{i,\tau,k} = 1$:
\begin{equation}
\text{Term I}
=
\left|\sum_{k=1}^{K} w_{i,\tau,k}(\ovar{\bar{E}_k} - \beta_k)\right|
\leq
\sum_{k=1}^{K} w_{i,\tau,k}\,|\ovar{\bar{E}_k} - \beta_k|
\leq
\max_{k}\,|\ovar{\bar{E}_k} - \beta_k|.
\end{equation}
From Stage 1, $\ovar{\bar{E}_k} \xrightarrow{a.s.} \beta_k$ for each $k$. Since $K$ is finite, the maximum over finitely many almost-sure convergences is also almost sure:
\begin{equation}
\max_{k=1,\ldots,K}\,|\ovar{\bar{E}_k} - \beta_k|
\;\xrightarrow{a.s.}\;
0
\quad \text{as } N_{\mathrm{tr}}\to\infty.
\end{equation}
In particular, this convergence holds in probability, so Term I $\xrightarrow{p} 0$.

\textbf{Bounding Term II.}
The soft weights $w_{i,\tau,k}$ concentrate around the bin $k^* = \arg\min_k |\hat{Y}_{i,\tau} - c_k|$ as $\sigma_c \to 0$. Formally, as $\sigma_c \to 0$, the Gaussian kernel degenerates to a point mass at $k^*$, so $w_{i,\tau,k} \to \mathbb{I}(k = k^*)$. Therefore:
\begin{equation}
\sum_{k=1}^{K} w_{i,\tau,k}\,\beta_k
\;\to\;
\beta_{k^*}
\quad \text{as } \sigma_c \to 0.
\end{equation}
Since $\hat{Y}_{i,\tau} \in \mathcal{B}_{k^*}$ by definition of $k^*$, the approximation bound from Stage 2 gives $|\beta_{k^*} - \bvar{\beta(\hat{Y}_{i,\tau})}| \leq \Lambda h$. As $h \to 0$ (i.e., $K \to \infty$ with $\max_k|c_{k+1}-c_k| \to 0$), we obtain Term II $\to 0$.

\textbf{Combining both bounds.}
Taking $N_{\mathrm{tr}} \to \infty$ and simultaneously $\sigma_c \to 0$, $K \to \infty$ with $h \to 0$, both terms vanish:
\begin{equation}
\left|\bvar{E^{p}_{i,\tau}} - \bvar{\beta(\hat{Y}_{i,\tau})}\right|
\leq
\max_{k}\,|\ovar{\bar{E}_k} - \beta_k|
+
\Lambda h
+
o(1)
\;\xrightarrow{p}\;
0,
\end{equation}
where the $o(1)$ term accounts for the residual soft-weight spread at finite $\sigma_c$, which vanishes as $\sigma_c \to 0$. This establishes:
\begin{equation}
\bvar{E^{p}_{i,\tau}}
\;\xrightarrow{p}\;
\bvar{\beta(\hat{Y}_{i,\tau})}.
\end{equation}

The convergence result directly implies that the calibrated output $\hat{Y}^{cal}_{i,\tau} = \hat{Y}_{i,\tau} - a_{i,\tau}\bvar{E^{p}_{i,\tau}}$ asymptotically removes the value-dependent bias when $a_{i,\tau} \to 1$. Specifically, the residual bias after calibration satisfies:
\begin{equation}
\mathbb{E}\!\left[\hat{Y}^{cal}_{i,\tau} - Y_{i,\tau} \;\middle|\; \hat{Y}_{i,\tau} = v\right]
=
\bvar{\beta(v)} - a_{i,\tau}\,\bvar{E^{p}_{i,\tau}}
\;\xrightarrow{p}\;
\bvar{\beta(v)} - \bvar{\beta(v)}
=
0,
\end{equation}
confirming that the RBC module is an asymptotically unbiased post-hoc calibrator. The context-aware amplitude $a_{i,\tau}$ learned by $\Phi_a$ further adapts this correction to local node and temporal statistics, providing entry-wise refinement beyond the global bin-level estimate.
\end{proof}

\newpage

\section{Additional Results}

\subsection{Analysis on Block Missingness}
\label{app:block_missing_analysis}

To support the motivation in the introduction, we provide additional evidence on the impact of different missing patterns.
We compare random missingness and block missingness under the same missing rate and report the relative MAE degradation over the complete-observation setting.

As shown in Fig.~\ref{fig:block_missing_appendix}, block missingness consistently causes larger degradation than random missingness across different backbones and datasets.
This trend is especially clear on NREL, where block missingness increases MAE by $26.9\%$ for IGNNK and $29.3\%$ for KITS, while random missingness has only a small effect or even leads to slight negative degradation.
Although the absolute degradation varies across architectures, the relative gap between block and random missingness remains consistent.

These results indicate that the difficulty does not come from a specific model design.
Instead, block missingness removes temporally contiguous observations and breaks local temporal trends, making the observed sensors less reliable as spatial anchors for inductive kriging.
This supports our focus on spatio-temporal kriging with block-missing observations on observed sensors.

\begin{figure}
    \centering
    \includegraphics[width=\textwidth]{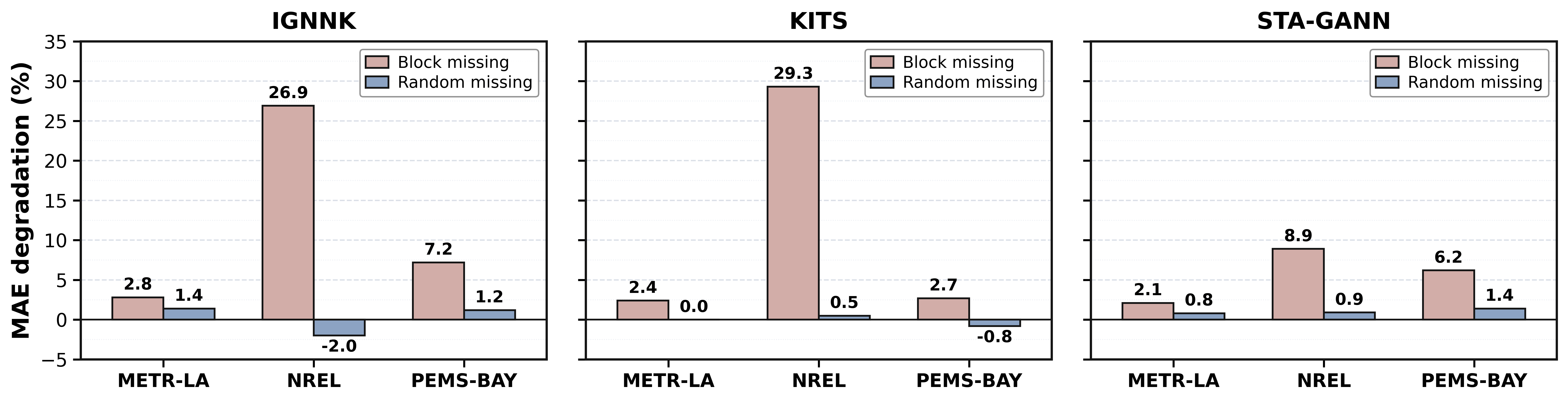}
    \caption{
    Relative MAE degradation under random and block missingness.
    }
    \label{fig:block_missing_appendix}
\end{figure}

\subsection{Objective Mismatch in the Imputation-first Pipeline}
\label{app:imputation_kriging_mismatch}

\paragraph{Experimental setup.}
We analyze whether better reconstruction of block-missing observations leads to better downstream kriging.
Given observed sensors with block missingness, we first apply an imputation method to complete the observed sensor sequences, and then feed the completed observations into an ISK model for kriging on unobserved sensors. We evaluate seven imputation methods, together with six ISK baseline.

\paragraph{Stage-1 imputation accuracy.}

Stage-1 imputation performance is visualized in Fig.~\ref{fig:stage1_imputation_mae}.
The results show that advanced deep imputation models indeed achieve substantially lower reconstruction errors than simple filling or interpolation methods under block missingness.
On METR-LA, SAITS achieves the lowest imputation MAE.
On PEMS-BAY, SAITS and ImputeFormer also achieve the best reconstruction accuracy.
These results confirm that the failure of the imputation-first pipeline is not because the imputation models are weak at the reconstruction task.

\begin{figure}
    \centering
    \includegraphics[width=\textwidth]{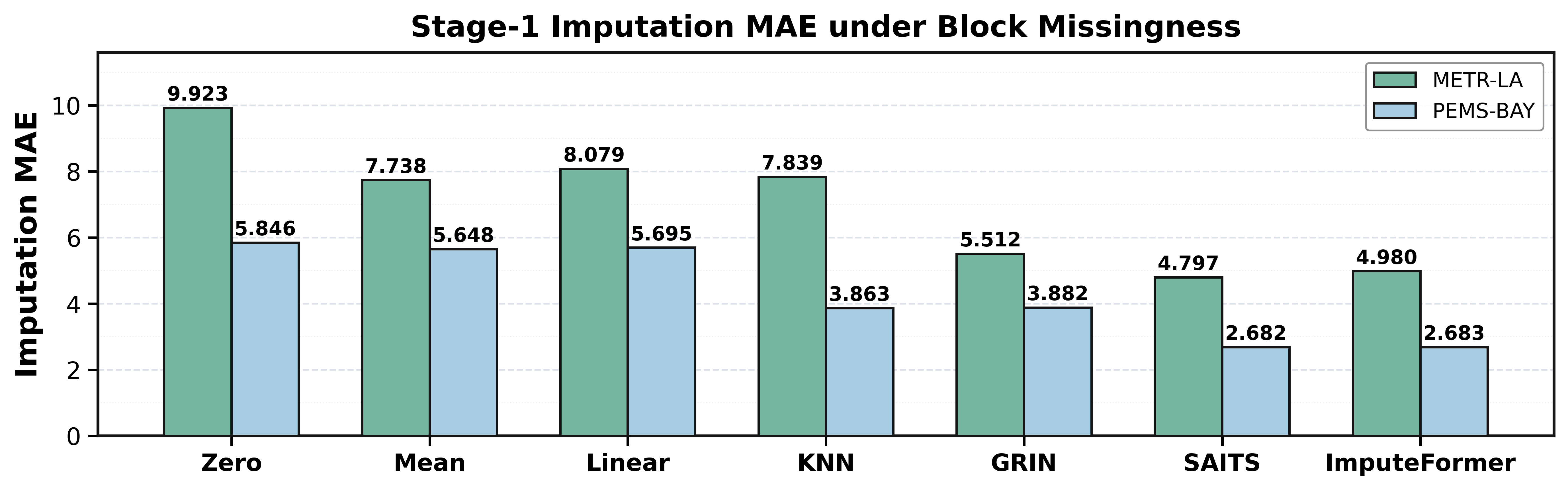}
    \caption{
    Stage-1 imputation MAE on observed sensors under block missingness.
    }
    \label{fig:stage1_imputation_mae}
\end{figure}

\paragraph{Stage-2 kriging performance.}
We then examine whether the Stage-1 reconstruction advantage transfers to downstream kriging.
For each dataset and backbone, Table~\ref{tab:stage2_kriging_summary} reports the best and worst imputation method according to the final kriging MAE.
The results show two clear patterns.
First, the best imputation method varies across backbones and datasets, indicating that there is no universally optimal imputation strategy for downstream kriging.
Second, strong reconstruction models do not consistently lead to the best kriging performance.

\begin{table*}[t]
\centering
\caption{Stage-2 kriging performance of the imputation-first pipeline under block missingness. We report the best and worst imputation method for each dataset-backbone pair.}
\label{tab:stage2_kriging_summary}
\small
\resizebox{\textwidth}{!}{
\begin{tabular}{llccccc}
\toprule
Dataset & Backbone & Best Method & Best MAE & Worst Method & Worst MAE \\
\midrule
\multirow{6}{*}{METR-LA}
& IGNNK & Linear & \textbf{6.230} & KNN & 6.434 \\
& KITS & Zero & \textbf{6.548} & ImputeFormer & 6.781 \\
& STGCN & Zero & \textbf{6.340} & Mean & 6.545 \\
& STA-GANN & Zero & \textbf{6.229} & ImputeFormer & 6.641 \\
& DARKFARSEER & Mean & \textbf{6.375} & Zero & 6.793 \\
& INCREASE-SP & SAITS & \textbf{6.201} & Zero & 6.392 \\
\midrule
\multirow{6}{*}{PEMS-BAY}
& IGNNK & Linear & \textbf{3.856} & ImputeFormer & 3.887 \\
& KITS & KNN & \textbf{3.962} & SAITS & 4.098 \\
& STGCN & Zero & \textbf{3.684} & Linear & 3.848 \\
& STA-GANN & GRIN & \textbf{3.788} & Zero & 3.840 \\
& DARKFARSEER & SAITS & \textbf{4.287} & KNN & 4.522 \\
& INCREASE-SP & Mean & \textbf{3.733} & Linear & 3.925 \\
\bottomrule
\end{tabular}
}
\end{table*}

\paragraph{Rank correlation between imputation and kriging.}
To further examine the alignment between Stage-1 imputation and Stage-2 kriging, we rank the imputation methods by their reconstruction MAE and compare them with their downstream kriging ranks under different ISK backbones.

As shown in Fig.~\ref{fig:imputation_kriging_rank_heatmap}, lower reconstruction error does not consistently correspond to better downstream kriging rank.
The average Spearman rank correlations are only -0.22 on METR-LA and -0.05 on PEMS-BAY, indicating weak or even negative alignment between the two stages.
This further confirms that reconstruction MAE on observed sensors is not a reliable surrogate objective for downstream inductive spatial kriging.

\begin{figure}
    \centering
    \includegraphics[width=\textwidth]{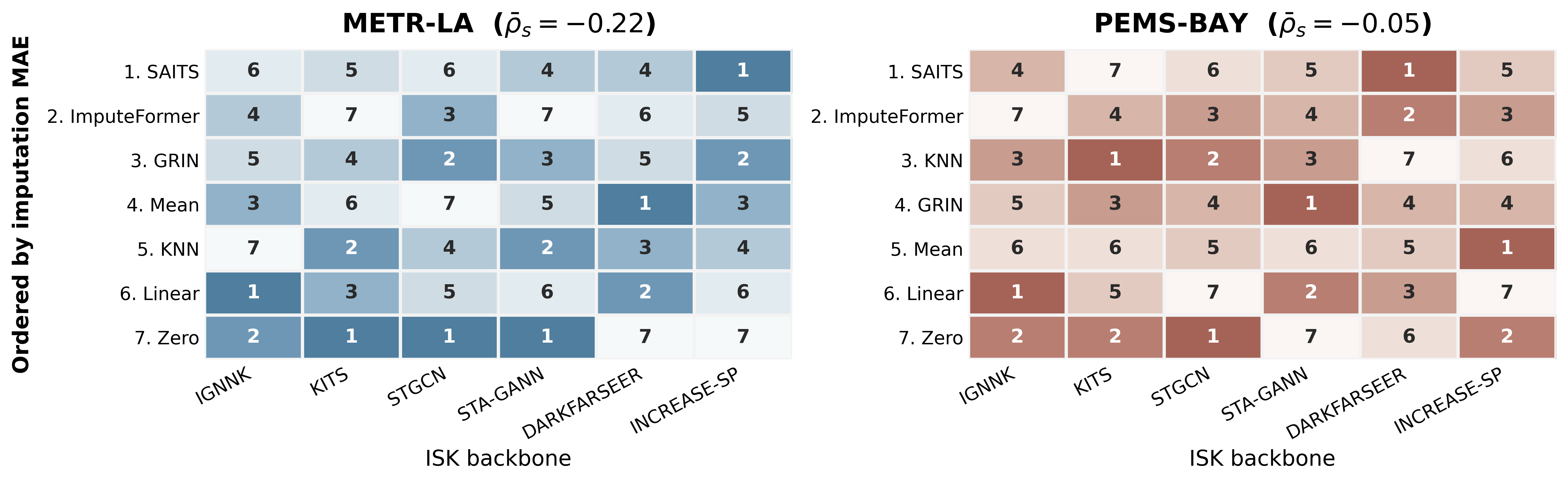}
    \caption{
    Comparison between Stage-1 reconstruction ranking and Stage-2 kriging ranking.
    For each dataset, rows list the imputation methods ordered by reconstruction MAE, and each cell shows the downstream kriging rank under a specific ISK backbone.
    }
    \label{fig:imputation_kriging_rank_heatmap}
\end{figure}

\paragraph{Discussion.}
These results reveal an objective mismatch in the imputation-first pipeline.
Although deep imputation models can substantially improve numerical reconstruction on observed sensors, this advantage does not consistently transfer to downstream kriging.
A likely reason is that block-missing observations are not merely missing values to be repaired; they also change the reliability of observed sensors as spatial anchors.
Once the imputed sequences are treated as fully observed inputs, the downstream ISK backbone cannot distinguish reliable measurements from reconstructed values, and may propagate imputation-induced bias to unobserved sensors.
This motivates modeling missingness and spatial inference jointly rather than optimizing imputation as an isolated preprocessing step.

\subsection{Bias Diagnosis under Block-wise Incomplete Observations}
\label{app:bias_diagnosis}

\paragraph{Bias metrics.}
We define the signed prediction error as
$
e_{i,t}=\hat{x}_{i,t}-x_{i,t},
$
where $e_{i,t}>0$ indicates overestimation and $e_{i,t}<0$ indicates underestimation.
For Stage-1 imputation, the evaluation set $\Omega_{\mathrm{imp}}$ contains the simulated missing entries on observed nodes with available ground truth.
For kriging, the evaluation set $\Omega_{\mathrm{kri}}$ contains the prediction entries on unobserved nodes.
The global mean error is computed as
$
\mathrm{GME}
=
\frac{1}{|\Omega|}
\sum_{(i,t)\in\Omega}
e_{i,t}.
$
To measure value-dependent bias, we partition the ground-truth values into quantile-based bins $\{\mathcal{B}_q\}_{q=1}^{Q}$ and compute the conditional bias as
$
\mathrm{CB}_{q}
=
\frac{1}{|\Omega_q|}
\sum_{(i,t)\in\Omega_q}
e_{i,t},
\quad
\Omega_q
=
\{(i,t)\in\Omega \mid x_{i,t}\in\mathcal{B}_q\}.
$

\paragraph{imputation bias.}
Table~\ref{tab:stage1_bias} reports the signed bias of different imputation methods on the simulated missing entries of observed nodes.
Although the global mean error is small for all methods, the conditional bias is much more evident.
In particular, all methods overestimate low-value regions and underestimate high-value regions, showing that imputation errors are value-dependent rather than uniformly distributed.

\begin{table}[t]
\centering
\caption{
Stage-1 imputation bias on simulated missing entries of observed nodes.
Positive values indicate overestimation, and negative values indicate underestimation.
}
\label{tab:stage1_bias}
\small
\begin{tabular}{lccccc}
\toprule
Method & GME & MAE & Low CB & Mid CB & High CB \\
\midrule
Zero & +0.37 & 9.92 & +15.33 & -4.81 & -9.40 \\
Mean & -0.18 & 7.74 & +9.60 & -4.51 & -5.63 \\
Linear & -0.38 & 8.08 & +7.49 & -4.25 & -4.38 \\
KNN & +0.67 & 7.84 & +9.28 & -2.05 & -5.21 \\
GRIN & +1.80 & 5.51 & +7.97 & -0.40 & -2.17 \\
SAITS & +1.38 & 4.80 & +7.08 & -1.04 & -1.89 \\
ImputeFormer & +1.27 & 4.98 & +6.32 & -0.98 & -1.53 \\
\bottomrule
\end{tabular}
\end{table}

\paragraph{Bias after kriging.}
We further examine whether the value-dependent bias is removed or propagated after spatial inference.
To separate the intrinsic bias of the kriging backbone from the additional bias induced by incomplete observations, we also report a no-missing reference setting, where the kriging model directly uses the original observed values without imputation.
Table~\ref{tab:stage2_bias_summary} summarizes the kriging bias under this reference setting and under block-wise incomplete observations with different imputers.

The no-missing reference already exhibits mild value-dependent bias, indicating that kriging itself can produce biased estimates across different value regimes.
However, the bias becomes substantially larger under block-wise incomplete observations.
In particular, the low-value overestimation and high-value underestimation are consistently preserved after kriging.
Moreover, although stronger imputers have smaller Stage-1 conditional bias, their downstream kriging bias becomes comparable to that of simpler imputers.
This suggests that spatial inference does not merely inherit imputation errors locally, but propagates and reshapes value-dependent bias through the kriging backbone.

\begin{table}[t]
\centering
\caption{
Kriging bias under the no-missing reference and block-wise incomplete observations.
The no-missing setting uses original observed values without imputation, while the imputation setting reports the range across seven imputation methods.
Positive values indicate overestimation, and negative values indicate underestimation.
}
\label{tab:stage2_bias_summary}
\small
\setlength{\tabcolsep}{5.5pt}
\renewcommand{\arraystretch}{1.15}
\begin{tabular}{lcccc}
\toprule
Backbone & Setting & GME & Low-value CB & High-value CB \\
\midrule
\multirow{2}{*}{DARKFARSEER}
& \quad No missing & 2.10 & 6.85 & -2.10 \\
\cmidrule(lr){2-5}
& \quad Imputation & [4.40, 4.60] & [12.16, 13.42] & [-4.25, -4.12] \\
\midrule

\multirow{2}{*}{IGNNK}
& \quad No missing & 1.20 & 5.35 & -1.85 \\
\cmidrule(lr){2-5}
& \quad Imputation & [2.49, 2.86] & [11.57, 12.20] & [-3.72, -3.42] \\
\midrule

\multirow{2}{*}{INCREASE-SP}
& \quad No missing & 1.35 & 5.80 & -1.90 \\
\cmidrule(lr){2-5}
& \quad Imputation & [2.47, 2.81] & [10.94, 12.39] & [-3.85, -3.40] \\
\midrule

\multirow{2}{*}{KITS}
& \quad No missing & 0.95 & 4.90 & -1.75 \\
\cmidrule(lr){2-5}
& \quad Imputation & [1.74, 2.92] & [10.85, 12.48] & [-4.53, -3.39] \\
\midrule

\multirow{2}{*}{STA-GANN}
& \quad No missing & 1.15 & 5.20 & -1.80 \\
\cmidrule(lr){2-5}
& \quad Imputation & [2.00, 2.67] & [10.85, 11.99] & [-4.42, -3.68] \\
\midrule

\multirow{2}{*}{STGCN}
& \quad No missing & 1.10 & 5.05 & -1.70 \\
\cmidrule(lr){2-5}
& \quad Imputation & [2.07, 2.90] & [10.70, 11.87] & [-4.17, -3.37] \\
\bottomrule
\end{tabular}
\end{table}

\paragraph{Discussion.}
The above results provide three observations.
First, the no-missing reference already shows non-negligible value-dependent bias, indicating that kriging backbones themselves are not fully unbiased across different value regimes.
Second, compared with the no-missing setting, block-wise incomplete observations consistently lead to larger low-value overestimation and high-value underestimation, showing that biased imputed values can further amplify the intrinsic bias of kriging.
Third, under the same backbone, the bias ranges across different imputers are relatively narrow, even though their Stage-1 imputation errors can be different.
This suggests that spatial inference may propagate and reshape value-dependent errors into similar residual patterns at unobserved nodes.

\subsection{Additional Results under Different Missing Settings}
\label{app:additional_results}

\paragraph{Results under Random Missingness}

To examine whether \UniSTOK is restricted to block-missing observations, we additionally evaluate it under random missingness. 
Different from block missingness, random missingness independently removes observations at scattered time steps, which weakens temporal continuity but does not create long consecutive gaps. 
As shown in Table~\ref{tab:random_missing_results}, \UniSTOK still improves different backbones on METR-LA, PEMS-BAY, and NREL-AL. 
This verifies that the proposed reliability-guided regulation is not tailored to a specific missing pattern, but can adapt to diverse observation corruption mechanisms.

\begin{table*}[t]
\centering
\caption{
  Performance comparison under random missingness across backbones and datasets.
  Lower is better.
  \textcolor{basegray}{Gray}: base model results;
  \textbf{Black}: results with our plug-in (\textbf{+Ours}).
  {\color{red}$\downarrow$\!x}: relative improvement (\%) of \textbf{+Ours} over Base.
}
\label{tab:random_missing_results}
\renewcommand{\tabcolsep}{3pt}
\renewcommand{\arraystretch}{1.05}
\resizebox{\textwidth}{!}{%
\begin{tabular}{ll r | ccc | ccc | cc | c}
\toprule
\multicolumn{3}{c|}{\multirow{2}{*}{\scalebox{1.05}{\textbf{Models}}}}
  & \multicolumn{3}{c|}{\textbf{METR-LA}}
  & \multicolumn{3}{c|}{\textbf{PEMS-BAY}}
  & \multicolumn{2}{c|}{\textbf{NREL-AL}}
  & \multirow{2}{*}{\shortstack{\textbf{Avg.}\\\textbf{Improv.(\%)}}} \\
\cmidrule(lr){4-6}\cmidrule(lr){7-9}\cmidrule(lr){10-11}
\multicolumn{3}{c|}{}
  & \textbf{MAE} & \textbf{RMSE} & \textbf{MAPE}
  & \textbf{MAE} & \textbf{RMSE} & \textbf{MAPE}
  & \textbf{MAE} & \textbf{RMSE}
  & \\
\midrule

\multirow{2}{*}{\shortstack[l]{\textbf{DCRNN}\\{\scriptsize ICLR'18}}}
  & \multirow{2}{*}{}
  & \textcolor{basegray}{Base}
  & \basenum{6.03} & \basenum{9.51} & \basenum{17.09}
  & \basenum{3.58} & \basenum{6.64} & \basenum{8.75}
  & \basenum{1.98} & \basenum{3.87}
  & \textcolor{basegray}{--} \\
  & & \textbf{+Ours}
  & \textbf{5.82}\imp{3.5} & \textbf{9.44}\imp{0.7} & \textbf{16.21}\imp{5.1}
  & \textbf{3.41}\imp{4.7} & \textbf{6.57}\imp{1.1} & \textbf{8.35}\imp{4.6}
  & \textbf{1.76}\imp{11.1} & \textbf{3.25}\imp{16.0}
  & \textbf{\color{red}5.9} \\
\midrule

\multirow{2}{*}{\shortstack[l]{\textbf{STGCN}\\{\scriptsize IJCAI'18}}}
  & \multirow{2}{*}{}
  & \textcolor{basegray}{Base}
  & \basenum{6.04} & \basenum{9.65} & \basenum{17.32}
  & \basenum{3.59} & \basenum{6.70} & \basenum{8.83}
  & \basenum{2.01} & \basenum{3.93}
  & \textcolor{basegray}{--} \\
  & & \textbf{+Ours}
  & \textbf{5.72}\imp{5.3} & \textbf{9.34}\imp{3.2} & \textbf{16.20}\imp{6.5}
  & \textbf{3.34}\imp{7.0} & \textbf{6.46}\imp{3.6} & \textbf{8.22}\imp{6.9}
  & \textbf{1.66}\imp{17.4} & \textbf{3.30}\imp{16.0}
  & \textbf{\color{red}8.2} \\
\midrule

\multirow{2}{*}{\shortstack[l]{\textbf{GMAN}\\{\scriptsize AAAI'20}}}
  & \multirow{2}{*}{}
  & \textcolor{basegray}{Base}
  & \basenum{6.37} & \basenum{11.10} & \basenum{17.62}
  & \basenum{3.92} & \basenum{7.34} & \basenum{8.99}
  & \basenum{1.69} & \basenum{3.13}
  & \textcolor{basegray}{--} \\
  & & \textbf{+Ours}
  & \textbf{5.91}\imp{7.2} & \textbf{9.39}\imp{15.4} & \textbf{16.55}\imp{6.1}
  & \textbf{3.68}\imp{6.1} & \textbf{6.74}\imp{8.2} & \textbf{8.61}\imp{4.2}
  & \textbf{1.42}\imp{16.0} & \textbf{2.97}\imp{5.1}
  & \textbf{\color{red}8.5} \\
\midrule

\multirow{2}{*}{\shortstack[l]{\textbf{IGNNK}\\{\scriptsize AAAI'21}}}
  & \multirow{2}{*}{}
  & \textcolor{basegray}{Base}
  & \basenum{5.94} & \basenum{9.63} & \basenum{16.88}
  & \basenum{3.63} & \basenum{6.60} & \basenum{8.72}
  & \basenum{2.04} & \basenum{4.15}
  & \textcolor{basegray}{--} \\
  & & \textbf{+Ours}
  & \textbf{5.36}\imp{9.8} & \textbf{8.74}\imp{9.2} & \textbf{15.67}\imp{7.2}
  & \textbf{3.41}\imp{6.1} & \textbf{6.54}\imp{0.9} & \textbf{8.19}\imp{6.1}
  & \textbf{1.76}\imp{13.7} & \textbf{3.29}\imp{20.7}
  & \textbf{\color{red}9.2} \\
\midrule

\multirow{2}{*}{\shortstack[l]{\textbf{SATCN}\\{\scriptsize Arxiv'21}}}
  & \multirow{2}{*}{}
  & \textcolor{basegray}{Base}
  & \basenum{6.01} & \basenum{9.53} & \basenum{16.93}
  & \basenum{3.57} & \basenum{6.73} & \basenum{8.74}
  & \basenum{1.60} & \basenum{3.02}
  & \textcolor{basegray}{--} \\
  & & \textbf{+Ours}
  & \textbf{5.72}\imp{4.8} & \textbf{9.26}\imp{2.8} & \textbf{16.34}\imp{3.5}
  & \textbf{3.52}\imp{1.4} & \textbf{6.64}\imp{1.3} & \textbf{8.62}\imp{1.4}
  & \textbf{1.34}\imp{16.2} & \textbf{2.83}\imp{6.3}
  & \textbf{\color{red}4.7} \\
\midrule

\multirow{2}{*}{\shortstack[l]{\textbf{INCREASE}\\{\scriptsize WWW'23}}}
  & \multirow{2}{*}{}
  & \textcolor{basegray}{Base}
  & \basenum{5.89} & \basenum{9.54} & \basenum{16.85}
  & \basenum{3.69} & \basenum{6.77} & \basenum{8.83}
  & \basenum{1.94} & \basenum{3.85}
  & \textcolor{basegray}{--} \\
  & & \textbf{+Ours}
  & \textbf{5.31}\imp{9.8} & \textbf{8.87}\imp{7.0} & \textbf{16.13}\imp{4.3}
  & \textbf{3.42}\imp{7.3} & \textbf{6.65}\imp{1.8} & \textbf{8.47}\imp{4.1}
  & \textbf{1.86}\imp{4.1} & \textbf{3.55}\imp{7.8}
  & \textbf{\color{red}5.8} \\
\midrule

\multirow{2}{*}{\shortstack[l]{\textbf{KITS}\\{\scriptsize AAAI'25}}}
  & \multirow{2}{*}{}
  & \textcolor{basegray}{Base}
  & \basenum{6.29} & \basenum{10.35} & \basenum{18.02}
  & \basenum{3.98} & \basenum{7.34} & \basenum{9.40}
  & \basenum{2.59} & \basenum{4.77}
  & \textcolor{basegray}{--} \\
  & & \textbf{+Ours}
  & \textbf{5.76}\imp{8.4} & \textbf{9.48}\imp{8.4} & \textbf{16.56}\imp{8.1}
  & \textbf{3.93}\imp{1.3} & \textbf{6.87}\imp{6.4} & \textbf{9.25}\imp{1.6}
  & \textbf{2.31}\imp{10.8} & \textbf{4.32}\imp{9.4}
  & \textbf{\color{red}6.8} \\
\midrule

\multirow{2}{*}{\shortstack[l]{\textbf{STA-GANN}\\{\scriptsize CIKM'25}}}
  & \multirow{2}{*}{}
  & \textcolor{basegray}{Base}
  & \basenum{5.98} & \basenum{9.60} & \basenum{16.55}
  & \basenum{3.61} & \basenum{6.68} & \basenum{8.79}
  & \basenum{1.48} & \basenum{2.84}
  & \textcolor{basegray}{--} \\
  & & \textbf{+Ours}
  & \textbf{5.54}\imp{7.4} & \textbf{9.01}\imp{6.1} & \textbf{15.96}\imp{3.6}
  & \textbf{3.43}\imp{5.0} & \textbf{6.36}\imp{4.8} & \textbf{8.34}\imp{5.1}
  & \textbf{1.37}\imp{7.4} & \textbf{2.69}\imp{5.3}
  & \textbf{\color{red}5.6} \\
\midrule

\multirow{2}{*}{\shortstack[l]{\textbf{DarkFarseer}\\{\scriptsize AAAI'26}}}
  & \multirow{2}{*}{}
  & \textcolor{basegray}{Base}
  & \basenum{6.17} & \basenum{9.88} & \basenum{18.24}
  & \basenum{4.25} & \basenum{8.07} & \basenum{11.18}
  & \basenum{2.05} & \basenum{3.96}
  & \textcolor{basegray}{--} \\
  & & \textbf{+Ours}
  & \textbf{5.51}\imp{10.7} & \textbf{9.44}\imp{4.5} & \textbf{16.39}\imp{10.1}
  & \textbf{4.04}\imp{4.9} & \textbf{7.59}\imp{5.9} & \textbf{10.21}\imp{8.7}
  & \textbf{1.87}\imp{8.8} & \textbf{3.34}\imp{15.7}
  & \textbf{\color{red}8.7} \\

\midrule
\rowcolor{blue!8}
\multicolumn{3}{c|}{\scalebox{1.02}{\textbf{Avg.\ Improv.\ (\%)}}}
  & \textbf{\color{red}7.4} & \textbf{\color{red}6.4} & \textbf{\color{red}6.0}
  & \textbf{\color{red}4.9} & \textbf{\color{red}3.8} & \textbf{\color{red}4.7}
  & \textbf{\color{red}11.7} & \textbf{\color{red}11.4}
  & \textbf{\color{red}7.0} \\
\bottomrule
\end{tabular}%
}
\end{table*}

\paragraph{Results without Missing Observations}

We also evaluate \UniSTOK in the no-missing setting, where all observations of available sensors are preserved. 
This setting removes the direct influence of missing-value imputation and thus isolates whether \UniSTOK can still benefit inductive kriging itself. 
As reported in Table~\ref{tab:no_missing_results}, \UniSTOK continues to improve the base backbones in most cases. 
This observation suggests that \UniSTOK not only mitigates missing-value corruption, but also calibrates the intrinsic residual bias caused by transferring patterns from observed nodes to unseen nodes.

\begin{table*}[t]
\centering
\caption{
  Performance comparison under the no-missing setting across backbones and datasets.
  Lower is better.
  \textcolor{basegray}{Gray}: base model results;
  \textbf{Black}: results with our plug-in (\textbf{+Ours}).
  {\color{red}$\downarrow$\!x}: relative improvement (\%) of \textbf{+Ours} over Base.
}
\label{tab:no_missing_results}
\renewcommand{\tabcolsep}{3pt}
\renewcommand{\arraystretch}{1.05}
\resizebox{\textwidth}{!}{%
\begin{tabular}{ll r | ccc | ccc | cc | c}
\toprule
\multicolumn{3}{c|}{\multirow{2}{*}{\scalebox{1.05}{\textbf{Models}}}}
  & \multicolumn{3}{c|}{\textbf{METR-LA}}
  & \multicolumn{3}{c|}{\textbf{PEMS-BAY}}
  & \multicolumn{2}{c|}{\textbf{NREL-AL}}
  & \multirow{2}{*}{\shortstack{\textbf{Avg.}\\\textbf{Improv.(\%)}}} \\
\cmidrule(lr){4-6}\cmidrule(lr){7-9}\cmidrule(lr){10-11}
\multicolumn{3}{c|}{}
  & \textbf{MAE} & \textbf{RMSE} & \textbf{MAPE}
  & \textbf{MAE} & \textbf{RMSE} & \textbf{MAPE}
  & \textbf{MAE} & \textbf{RMSE}
  & \\
\midrule

\multirow{2}{*}{\shortstack[l]{\textbf{DCRNN}\\{\scriptsize ICLR'18}}}
  & \multirow{2}{*}{}
  & \textcolor{basegray}{Base}
  & \basenum{6.01} & \basenum{9.47} & \basenum{17.02}
  & \basenum{3.54} & \basenum{6.61} & \basenum{8.67}
  & \basenum{1.90} & \basenum{3.72}
  & \textcolor{basegray}{--} \\
  & & \textbf{+Ours}
  & \textbf{5.75}\imp{4.3} & \textbf{9.42}\imp{0.5} & \textbf{16.13}\imp{5.2}
  & \textbf{3.38}\imp{4.5} & \textbf{6.53}\imp{1.2} & \textbf{8.36}\imp{3.6}
  & \textbf{1.74}\imp{8.4} & \textbf{3.24}\imp{12.9}
  & \textbf{\color{red}5.1} \\
\midrule

\multirow{2}{*}{\shortstack[l]{\textbf{STGCN}\\{\scriptsize IJCAI'18}}}
  & \multirow{2}{*}{}
  & \textcolor{basegray}{Base}
  & \basenum{6.07} & \basenum{9.63} & \basenum{17.29}
  & \basenum{3.59} & \basenum{6.71} & \basenum{8.86}
  & \basenum{1.93} & \basenum{3.88}
  & \textcolor{basegray}{--} \\
  & & \textbf{+Ours}
  & \textbf{5.69}\imp{6.3} & \textbf{9.26}\imp{3.8} & \textbf{16.11}\imp{6.8}
  & \textbf{3.32}\imp{7.5} & \textbf{6.44}\imp{4.0} & \textbf{8.13}\imp{8.2}
  & \textbf{1.62}\imp{16.1} & \textbf{3.29}\imp{15.2}
  & \textbf{\color{red}8.5} \\
\midrule

\multirow{2}{*}{\shortstack[l]{\textbf{GMAN}\\{\scriptsize AAAI'20}}}
  & \multirow{2}{*}{}
  & \textcolor{basegray}{Base}
  & \basenum{6.32} & \basenum{10.74} & \basenum{17.55}
  & \basenum{3.94} & \basenum{7.40} & \basenum{8.95}
  & \basenum{1.65} & \basenum{3.08}
  & \textcolor{basegray}{--} \\
  & & \textbf{+Ours}
  & \textbf{5.90}\imp{6.6} & \textbf{9.35}\imp{12.9} & \textbf{16.49}\imp{6.0}
  & \textbf{3.67}\imp{6.9} & \textbf{6.77}\imp{8.5} & \textbf{8.58}\imp{4.1}
  & \textbf{1.40}\imp{15.2} & \textbf{2.95}\imp{4.2}
  & \textbf{\color{red}8.1} \\
\midrule

\multirow{2}{*}{\shortstack[l]{\textbf{IGNNK}\\{\scriptsize AAAI'21}}}
  & \multirow{2}{*}{}
  & \textcolor{basegray}{Base}
  & \basenum{5.92} & \basenum{9.60} & \basenum{16.76}
  & \basenum{3.59} & \basenum{6.53} & \basenum{8.73}
  & \basenum{1.99} & \basenum{4.06}
  & \textcolor{basegray}{--} \\
  & & \textbf{+Ours}
  & \textbf{5.33}\imp{10.0} & \textbf{8.64}\imp{10.0} & \textbf{15.61}\imp{6.9}
  & \textbf{3.40}\imp{5.3} & \textbf{6.44}\imp{1.4} & \textbf{8.07}\imp{7.6}
  & \textbf{1.72}\imp{13.6} & \textbf{3.13}\imp{22.9}
  & \textbf{\color{red}9.7} \\
\midrule

\multirow{2}{*}{\shortstack[l]{\textbf{SATCN}\\{\scriptsize Arxiv'21}}}
  & \multirow{2}{*}{}
  & \textcolor{basegray}{Base}
  & \basenum{6.03} & \basenum{9.56} & \basenum{16.98}
  & \basenum{3.52} & \basenum{6.65} & \basenum{8.70}
  & \basenum{1.54} & \basenum{2.97}
  & \textcolor{basegray}{--} \\
  & & \textbf{+Ours}
  & \textbf{5.71}\imp{5.3} & \textbf{9.15}\imp{4.3} & \textbf{16.22}\imp{4.5}
  & \textbf{3.47}\imp{1.4} & \textbf{6.61}\imp{0.6} & \textbf{8.59}\imp{1.3}
  & \textbf{1.31}\imp{14.9} & \textbf{2.66}\imp{10.4}
  & \textbf{\color{red}5.3} \\
\midrule

\multirow{2}{*}{\shortstack[l]{\textbf{INCREASE}\\{\scriptsize WWW'23}}}
  & \multirow{2}{*}{}
  & \textcolor{basegray}{Base}
  & \basenum{5.87} & \basenum{9.52} & \basenum{16.75}
  & \basenum{3.67} & \basenum{6.74} & \basenum{8.76}
  & \basenum{1.88} & \basenum{3.75}
  & \textcolor{basegray}{--} \\
  & & \textbf{+Ours}
  & \textbf{5.29}\imp{9.9} & \textbf{8.74}\imp{8.2} & \textbf{15.90}\imp{5.1}
  & \textbf{3.38}\imp{7.9} & \textbf{6.54}\imp{3.0} & \textbf{8.29}\imp{5.4}
  & \textbf{1.83}\imp{2.7} & \textbf{3.43}\imp{8.5}
  & \textbf{\color{red}6.3} \\
\midrule

\multirow{2}{*}{\shortstack[l]{\textbf{KITS}\\{\scriptsize AAAI'25}}}
  & \multirow{2}{*}{}
  & \textcolor{basegray}{Base}
  & \basenum{6.23} & \basenum{10.11} & \basenum{17.49}
  & \basenum{3.96} & \basenum{7.36} & \basenum{9.43}
  & \basenum{2.45} & \basenum{4.41}
  & \textcolor{basegray}{--} \\
  & & \textbf{+Ours}
  & \textbf{5.71}\imp{8.3} & \textbf{9.42}\imp{6.8} & \textbf{16.38}\imp{6.3}
  & \textbf{3.92}\imp{1.0} & \textbf{6.91}\imp{6.1} & \textbf{9.18}\imp{2.7}
  & \textbf{2.14}\imp{12.7} & \textbf{4.16}\imp{5.7}
  & \textbf{\color{red}6.2} \\
\midrule

\multirow{2}{*}{\shortstack[l]{\textbf{STA-GANN}\\{\scriptsize CIKM'25}}}
  & \multirow{2}{*}{}
  & \textcolor{basegray}{Base}
  & \basenum{5.93} & \basenum{9.51} & \basenum{16.40}
  & \basenum{3.59} & \basenum{6.62} & \basenum{8.71}
  & \basenum{1.45} & \basenum{2.73}
  & \textcolor{basegray}{--} \\
  & & \textbf{+Ours}
  & \textbf{5.47}\imp{7.8} & \textbf{8.88}\imp{6.6} & \textbf{15.92}\imp{2.9}
  & \textbf{3.43}\imp{4.5} & \textbf{6.35}\imp{4.1} & \textbf{8.35}\imp{4.1}
  & \textbf{1.32}\imp{9.0} & \textbf{2.49}\imp{8.8}
  & \textbf{\color{red}6.0} \\
\midrule

\multirow{2}{*}{\shortstack[l]{\textbf{DarkFarseer}\\{\scriptsize AAAI'26}}}
  & \multirow{2}{*}{}
  & \textcolor{basegray}{Base}
  & \basenum{6.15} & \basenum{9.85} & \basenum{18.17}
  & \basenum{4.20} & \basenum{7.97} & \basenum{10.72}
  & \basenum{1.98} & \basenum{3.87}
  & \textcolor{basegray}{--} \\
  & & \textbf{+Ours}
  & \textbf{5.49}\imp{10.7} & \textbf{9.36}\imp{5.0} & \textbf{16.28}\imp{10.4}
  & \textbf{4.01}\imp{4.5} & \textbf{7.51}\imp{5.8} & \textbf{10.11}\imp{5.7}
  & \textbf{1.87}\imp{5.6} & \textbf{3.32}\imp{14.2}
  & \textbf{\color{red}7.7} \\

\midrule
\rowcolor{blue!8}
\multicolumn{3}{c|}{\scalebox{1.02}{\textbf{Avg.\ Improv.\ (\%)}}}
  & \textbf{\color{red}7.7} & \textbf{\color{red}6.5} & \textbf{\color{red}6.0}
  & \textbf{\color{red}4.8} & \textbf{\color{red}3.9} & \textbf{\color{red}4.7}
  & \textbf{\color{red}10.9} & \textbf{\color{red}11.4}
  & \textbf{\color{red}7.0} \\
\bottomrule
\end{tabular}%
}
\end{table*}

\paragraph{Results on Real-world Block-missing Data}

We further evaluate \UniSTOK on NZ-Highway, a real-world traffic volume dataset with naturally occurring block-wise missing observations.
Unlike synthetic block missingness, NZ-Highway contains original missing records from highway monitoring systems, making it a practical test case for incomplete spatio-temporal kriging.
As shown in Table~\ref{tab:nz_highway_results}, \UniSTOK improves most backbones on both MAE and RMSE, although the gains are smaller than those on standard benchmark datasets.
This suggests that \UniSTOK remains effective under real-world block-wise missing observations, while the high original missing rate and dataset-specific traffic dynamics make the task more challenging.

\begin{table}[t]
\centering
\caption{
Performance comparison on NZ-Highway with naturally occurring block-wise missing observations.
Lower is better.
  \textcolor{basegray}{Gray}: base model results;
  \textbf{Black}: results with our plug-in (\textbf{+Ours}).
  {\color{red}$\downarrow$\!x}: relative improvement (\%) of \textbf{+Ours} over Base.
}
\label{tab:nz_highway_results}
\vspace{3pt}
\tiny
\setlength{\tabcolsep}{3pt}
\renewcommand{\arraystretch}{1.02}
\resizebox{0.92\linewidth}{!}{%
\begin{tabular}{ll ccccccccc}
\toprule
\textbf{Metric} & \textbf{Setting}
& \textbf{DCRNN}
& \textbf{STGCN}
& \textbf{GMAN}
& \textbf{IGNNK}
& \textbf{SATCN}
& \textbf{INCREASE}
& \textbf{KITS}
& \textbf{STA-GANN}
& \textbf{DarkFarseer} \\
\midrule
\multirow{2}{*}{\textbf{MAE}}
& \textcolor{basegray}{Base}
& \basenum{74.23} & \basenum{78.52} & \basenum{97.37} & \basenum{78.55} & \basenum{79.36} & \basenum{76.28} & \basenum{85.09} & \basenum{76.07} & \basenum{99.23} \\
& \textbf{+Ours}
& \textbf{72.60}\imp{2.2} & \textbf{76.99}\imp{1.9} & \textbf{93.12}\imp{4.4} & \textbf{77.60}\imp{1.2} & \textbf{78.27}\imp{1.4} & \textbf{74.85}\imp{1.9} & \textbf{78.16}\imp{8.1} & \textbf{75.85}\imp{0.3} & \textbf{97.86}\imp{1.4} \\
\midrule
\multirow{2}{*}{\textbf{RMSE}}
& \textcolor{basegray}{Base}
& \basenum{112.59} & \basenum{114.85} & \basenum{143.81} & \basenum{115.29} & \basenum{115.38} & \basenum{115.37} & \basenum{130.74} & \basenum{113.17} & \basenum{146.18} \\
& \textbf{+Ours}
& \textbf{111.46}\imp{1.0} & \textbf{113.40}\imp{1.3} & \textbf{138.04}\imp{4.0} & \textbf{114.50}\imp{0.7} & \textbf{113.96}\imp{1.2} & \textbf{110.71}\imp{4.0} & \textbf{117.22}\imp{10.3} & \textbf{110.63}\imp{2.2} & \textbf{143.67}\imp{1.7} \\
\bottomrule
\end{tabular}%
}
\end{table}

\subsection{Additional Ablation Results}
\label{app:additional_ablation}

For module-level ablation analysis, we report the average relative degradation over the full model to make results comparable across datasets with different value scales.
For each dataset $\mathcal{D}_i$, the degradation of metric $m$ is computed as
\begin{equation}
D_{m}^{i}
=
\frac{
m^{i}_{\mathrm{variant}}
-
m^{i}_{\mathrm{Full}}
}{
m^{i}_{\mathrm{Full}}
}
\times 100\% .
\end{equation}
The final degradation score is averaged over datasets:
\begin{equation}
\overline{D}_{m}
=
\frac{1}{|\mathcal{D}_{m}|}
\sum_{i=1}^{|\mathcal{D}_{m}|}
D_{m}^{i},
\end{equation}
where $\mathcal{D}_{m}$ denotes the datasets with available results for metric $m$.
For MAE and RMSE, $\mathcal{D}_{m}$ contains all three datasets; for MAPE, it contains the datasets where MAPE is reported.

Figures~\ref{fig:ablation_appendix_all} reports the module-level ablation results under MAE, RMSE, and MAPE. The variants are organized into three groups.
The RSR variants remove temporal reliability, spatial reliability, or residual compensation from the reliability-based signal regulation module.
The adaptive fusion variants replace the full gated fusion with single-branch or non-gated alternatives.
The RBC variants remove context amplitude, bin-balanced learning, or soft-bin assignment from residual bias calibration.

\begin{figure*}[t]
    \centering

    \vspace{0.2em}
    {\small \textbf{MAE}}
    \vspace{0.2em}
    
    \includegraphics[width=\textwidth]{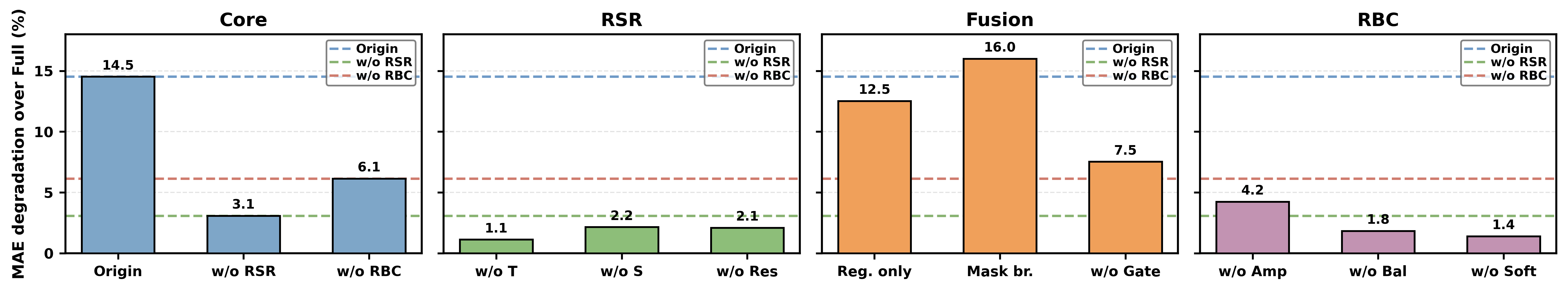}

    \vspace{0.6em}
    {\small \textbf{RMSE}}
    \vspace{0.2em}
    
    \includegraphics[width=\textwidth]{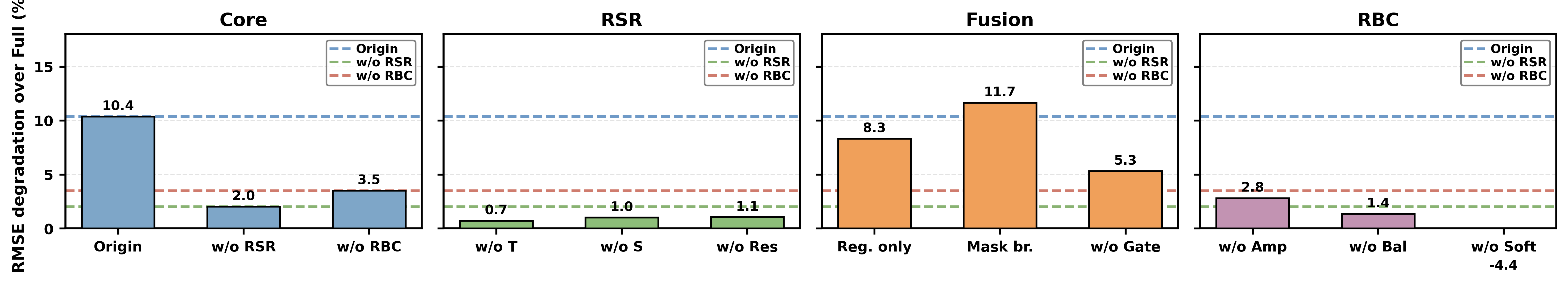}

    \vspace{0.6em}
    {\small \textbf{MAPE}}
    \vspace{0.2em}
    
    \includegraphics[width=\textwidth]{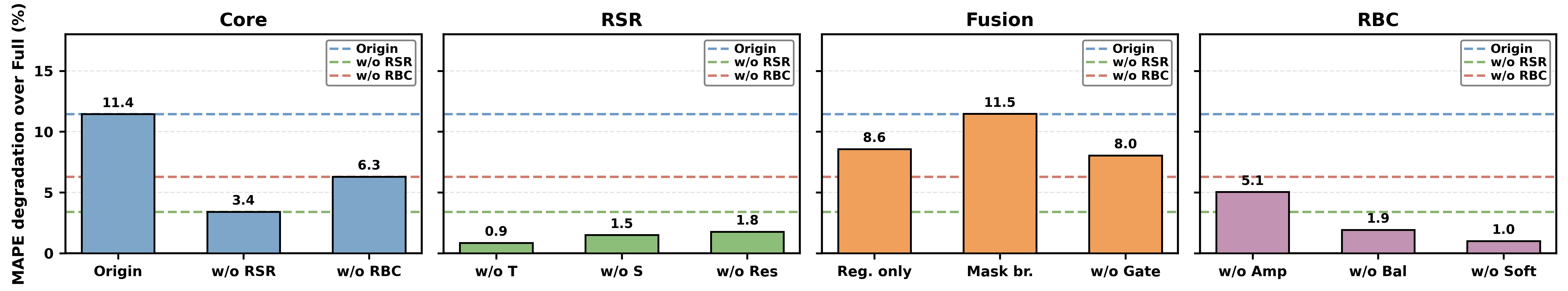}
    
    \caption{Module-level ablation results based on the average relative degradation over the full model. The RSR variants include \textit{w/o temporal reliability}, \textit{w/o spatial reliability}, and \textit{w/o residual compensation}. The fusion variants include \textit{regulated-only}, \textit{M as second branch}, and \textit{w/o gate}. The RBC variants include \textit{w/o context amplitude}, \textit{w/o bin-balanced loss}, and \textit{hard-bin RBC}.}

    \label{fig:ablation_appendix_all}
\end{figure*}

The results provide consistent evidence for the effectiveness of each internal design.
For RSR, removing temporal reliability, spatial reliability, or residual compensation increases degradation, showing that reliable input regulation requires both reliability estimation and learnable compensation.
For adaptive fusion, single-branch and non-gated alternatives are consistently worse than the full gated fusion, indicating that the model needs to adaptively balance the original and regulated branches.
For RBC, removing context amplitude, bin-balanced learning, or soft-bin assignment also degrades performance, confirming that residual calibration should be both value-aware and distribution-aware.
Overall, these fine-grained results further support the complementary roles of reliability-based signal regulation, adaptive branch fusion, and residual bias calibration.

\subsection{Additional Hyperparameter Sensitivity Results}
\label{app:hyperparameter_sensitivity}

Fig.~\ref{fig:hyperparameter_appendix} reports additional hyperparameter sensitivity results on PEMS-BAY and NREL.
The analysis follows the same setting as the main text, using IGNNK as the backbone and varying one hyperparameter at a time.
On PEMS-BAY, the performance changes only mildly across different values of $K$, $\eta$, and $\alpha$, indicating that the selected setting is not a fragile optimum.
A moderate residual bound gives slightly better RMSE, while overly large correction does not provide consistent gains, suggesting that RBC should be applied in a bounded manner.
For the regulation strength $\alpha$, moderate values also perform competitively, which supports the use of conservative RSR modulation.

On NREL, the curves show slightly larger fluctuations than those on PEMS-BAY, especially for the residual-bin setting.
This is likely because solar generation contains sharper value changes and a larger proportion of low-value regimes, making residual prototypes more sensitive to the bin partition.
Nevertheless, the overall performance remains within a narrow range under most settings.
The results further indicate that \UniSTOK does not rely on a single highly tuned hyperparameter configuration, while moderate reliability regulation and bounded residual calibration are generally more stable choices.

\begin{figure*}[t]
    \centering
    \begin{subfigure}{\textwidth}
        \centering
        \includegraphics[width=\textwidth]{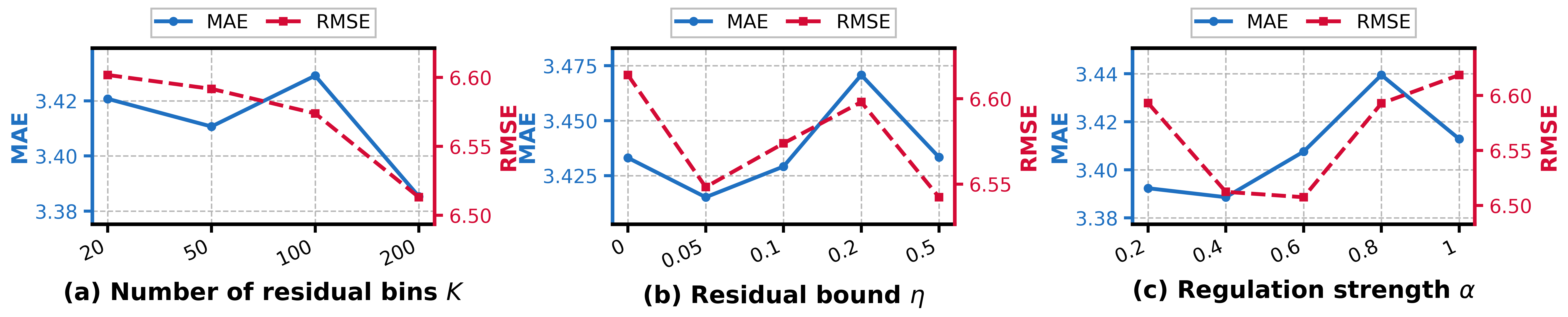}
        \caption{PEMS-BAY}
        \label{fig:hyperparameter_pems}
    \end{subfigure}
    
    \vspace{0.6em}
    
    \begin{subfigure}{\textwidth}
        \centering
        \includegraphics[width=\textwidth]{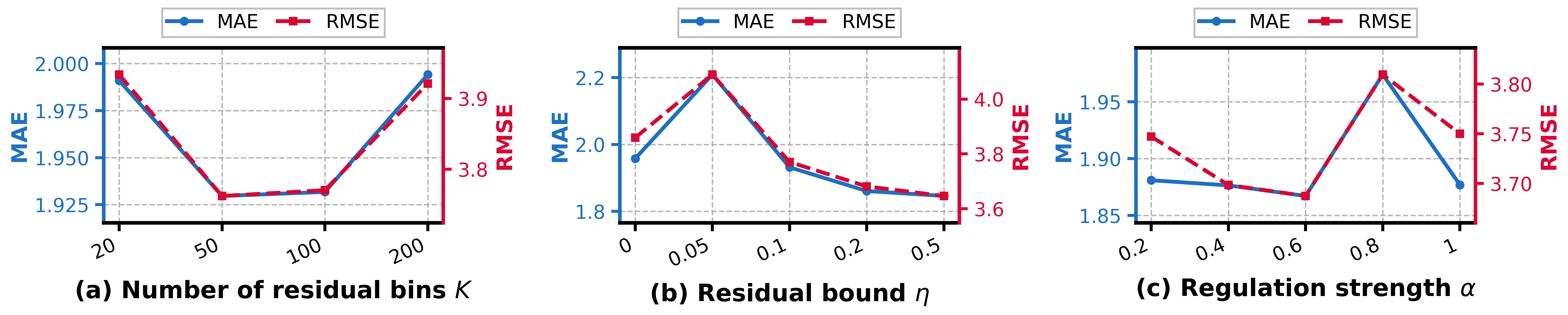}
        \caption{NREL}
        \label{fig:hyperparameter_nrel}
    \end{subfigure}
    \caption{
    Additional hyperparameter sensitivity results of \UniSTOK with IGNNK as the backbone.
    We analyze the number of residual bins $K$ in RBC, the residual correction bound $\eta$, and the reliability regulation strength $\alpha$ in RSR.
    Lower MAE and RMSE indicate better performance.
    }
    \label{fig:hyperparameter_appendix}
\end{figure*}

\subsection{Visualization of RSR and Adaptive Fusion}
\label{app:rsr_fusion_visualization}

Fig.~\ref{fig:rsr_fusion_appendix} provides additional visualizations of the RSR module and the adaptive fusion mechanism on four representative samples.
Each column corresponds to one sample, while the three rows show the observation mask $M$, the RSR-induced input difference $\widetilde{X}-X$, and the adaptive fusion gate $G$, respectively.

The mask maps show heterogeneous missing structures across samples, including scattered missing entries and more continuous sparse regions.
The maps of $\widetilde{X}-X$ contain both positive and negative adjustments, indicating that RSR performs bidirectional source-side regulation rather than uniformly amplifying the input.
In samples with more sparse or structured missing regions, stronger negative adjustments can be observed, suggesting suppression of potentially unreliable source information.
In other regions, localized positive adjustments indicate that RSR can also enhance useful source evidence when appropriate.
The gate maps further show structured and localized activations, indicating that adaptive fusion does not uniformly rely on the RSR-enhanced branch.
Instead, it selectively increases the contribution of the regulated branch in specific node-time regions.

Overall, these visualizations support the intended behavior of the proposed design: RSR regulates the source-side input according to reliability, and adaptive fusion determines when the RSR-enhanced branch should contribute to the final kriging prediction.

\begin{figure*}[t]
    \centering
    \includegraphics[width=\textwidth]{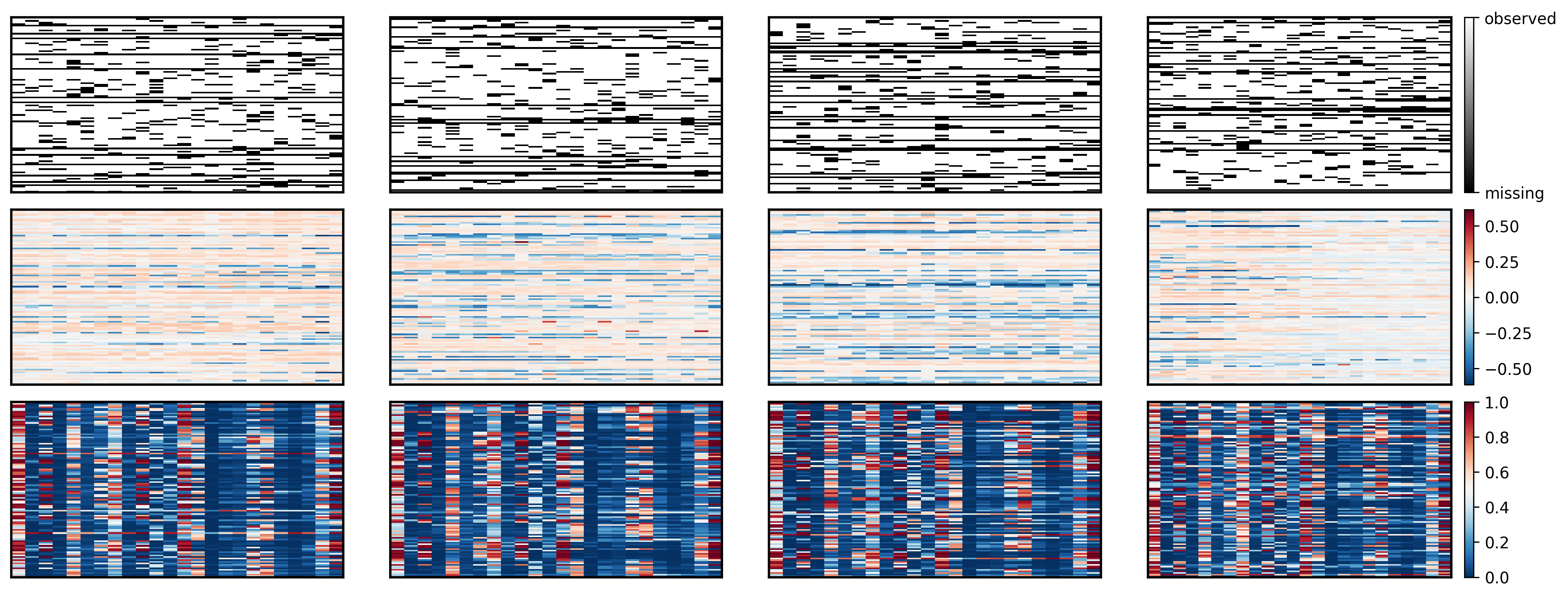}
    \caption{
    Additional visualization of RSR and adaptive fusion on four representative samples.
    Each column corresponds to one sample.
    From top to bottom, the three rows show the observation mask $M$, the RSR-induced input difference $\widetilde{X}-X$, and the adaptive fusion gate $G$.
    Larger $G$ indicates stronger reliance on the RSR-enhanced branch.
    }
    \label{fig:rsr_fusion_appendix}
\end{figure*}

\subsection{Visualization of Residual Bias Calibration}
\label{app:rbc_visualization}

Figure~\ref{fig:rbc_distribution_case} provides a distribution-level case study of RBC.
Before calibration, the predictions are overly concentrated in a narrow high-value range, indicating clear value-dependent residual bias and reduced distributional diversity.
After RBC, the prediction distribution becomes closer to the ground-truth distribution, with the excessive concentration alleviated.
This shows that RBC corrects residual bias instead of applying a uniform global offset.
The learned amplitude distribution further explains this behavior.
Most amplitudes are close to zero, while only a small fraction of entries receive stronger correction.
This indicates that RBC performs selective and conservative calibration: it preserves reliable predictions and mainly adjusts entries with stronger systematic bias.

\begin{figure}[h]
    \centering
    \includegraphics[width=\textwidth]{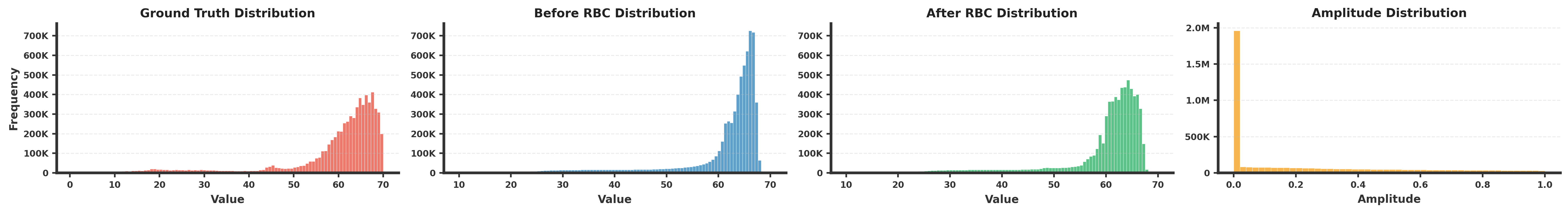}
    \caption{
    From left to right: ground-truth values, predictions before RBC, predictions after RBC, and learned correction amplitudes.
    }
    \label{fig:rbc_distribution_case}
\end{figure}

\end{document}